\documentclass{article}
\pdfoutput=1
\usepackage[preprint]{nips_2018}

\usepackage{adjustbox}

\usepackage{amsmath} % assumes amsmath package installed
\usepackage{amssymb}  % assumes amsmath package installed
\usepackage{xspace} % insert space when needed.
\usepackage{algorithm,algorithmic}

\usepackage{cite}
\usepackage{xcolor}
\usepackage{graphicx}
\usepackage{subfigure}
\usepackage{multirow}
\usepackage{amsthm}

\DeclareFontFamily{U}{wncy}{}
\DeclareFontShape{U}{wncy}{m}{n}{<->wncyr10}{}
\DeclareSymbolFont{mcy}{U}{wncy}{m}{n}
\DeclareMathSymbol{\Sh}{\mathord}{mcy}{"58} 
\newcommand{\minus}{\scalebox{0.75}[1.0]{$-$}}

\newtheorem{theo}{Theorem}
\newtheorem{lem}{Lemma}
\newtheorem{rem}{Remark}
\newcommand{\ie}{i\/.\/e\/.,\/~}
\newcommand{\eg}{e\/.\/g\/.,\/~}
\newcommand{\cf}{cf\/.\/~}
\newcommand{\fig}{Fig\/.\/~}
\newcommand{\alg}{Alg\/.\/~}
\newcommand{\sect}{Sec\/.\/~}

\newcommand*{\diffeo}{% 
  \mathrel{\vcenter{\offinterlineskip
  \hbox{$\sim$}\vskip-.35ex\hbox{$\sim$}\vskip-.35ex\hbox{$\sim$}}}}

\newcommand{\mieds}{L-ICDS\xspace}

\DeclareMathOperator*{\argmin}{arg\,min}

\newcommand{\icds}{ICDS\xspace}
\newcommand{\licds}{L-ICDS\xspace}

\usepackage[utf8]{inputenc} % allow utf-8 input
\usepackage[T1]{fontenc}    % use 8-bit T1 fonts
\usepackage{hyperref}       % hyperlinks
\usepackage{url}            % simple URL typesetting
\usepackage{booktabs}       % professional-quality tables
\usepackage{amsfonts}       % blackboard math symbols
\usepackage{nicefrac}       % compact symbols for 1/2, etc.
\usepackage{microtype}      % microtypography

\title{A Local Information Criterion for Dynamical Systems}

\author{Arash Mehrjou$^\dagger$\\
Department of Empirical Inference\\
Max Planck Institute for Intelligent Systems\\
\texttt{arash.mehrjou@tuebingen.mpg.de}\\
\And
Friedrich Solowjow$^\dagger$\\
Intelligent Control Systems Group \\
Max Planck Institute for Intelligent Systems\\
\texttt{fsolowjow@is.mpg.de}\\
\And
Sebastian Trimpe\\
Intelligent Control Systems Group\\
Max Planck Institute for Intelligent Systems\\
\texttt{trimpe@is.mpg.de}\\
\And
Bernhard Sch\"{o}lkopf\\
Department of Empirical Inference\\
Max Planck Institute for Intelligent Systems\\
\texttt{bs@tuebingen.mpg.de} \\
}

\begin{document}

\maketitle

\begin{abstract}
Encoding a sequence of observations is an essential task with many applications. The encoding can become highly efficient when the observations are generated by a dynamical system. A dynamical system imposes regularities on the observations that can be leveraged to achieve a more efficient code. We propose a method to encode a given or learned dynamical system. Apart from its application for encoding a sequence of observations, we propose to use the compression achieved by this encoding as a criterion for model selection. Given a dataset, different learning algorithms result in different models. But not all learned models are equally good. We show that the proposed encoding approach can be used to choose the learned model which is closer to the true underlying dynamics. We provide experiments for both encoding and model selection, and theoretical results that shed light on why the approach works.
\end{abstract}

\section{Introduction}
Objects are of various complexities in nature. A round stone looks simpler than a convoluted rough piece of rock; a constant beep-like sound is simpler than an orchestra. We humans have internal ideas about how complex are objects. The complexity can also be defined for abstract objects such as mathematical creatures. The focus of this paper is on the complexity of dynamical systems that model the laws of nature~\citep{newton1833philosophiae}. To our eyes, a dynamical system is nothing more than a temporal sequence of observations.  We might use the data sequence to infer a model.  But what is the better representation of the dynamical system -- the data or the model, and which model should we use?  In this paper, we take a closer look at
%we take a closer look at the description of a dynamical system; 
%we propose an 
efficient encoding of dynamical systems and, based on that, propose a model selection criterion with practical use in empirical inference.
%in the presence and absence of a model and based on that, we propose a model selection criterion with practical use in empirical inference.

%To ease the presentation, 
For illustration, assume the following scenario: Alice and Bob are friends and they are talking over the phone. Alice is watching a dynamical system and wants to share her experience with Bob. Alice knows what Bob knows about the nature, math, etc. She is watching a temporal sequence of observations $\mathcal{S}=[x(1), x(2), \ldots]$ caused by an underlying mathematical expression $\dot{x}(t)=f(x(t))$. 
Unfortunately, the transmission cord from Alice to Bob charges her for every voltage pulse. Therefore, Alice would like to transmit her experience to Bob with the least phone cost. Due to the physical constraints, Alice can observe samples from the model with sampling frequency $f_s=1/T_s$ where $T_s$ is the time interval between two consecutive observations. Assume the phone call starts at time $t_0$ and Alice can describe each of her observations with $m$ bits. One trivial solution is that Alice talks constantly with Bob and tells him every observation at each time instant $\{t_0, t_0+T_s, t_0+2T_s, \ldots\, t_0+nT_s\}$ for an indeterminate amount of time. Despite its simplicity, this approach will cost Alice a horrible amount $mn$ that increases without bound as $n\to \infty$. More cleverly, Alice can use her prior assumptions about nature and her belief that her observations are not totally random. Hence, she is able to infer the underlying dynamics by a \emph{nonparametric} model $\hat{f}$ from her observations $\mathcal{S}$. 
%This means that she infers the function $\hat{f}$ from observations $\mathcal{S}$. 
Assume this model is chosen from a hypothesis set $\mathcal{H}$ and both Alice and Bob agree on the members of $\mathcal{H}$. Thus, Alice only needs to inform Bob about the initial state $\hat{x}(0)=x(0)$ of the system and the model $\hat{f}$ she has learned about the dynamics. Bob can reconstruct the sequence $\hat{\mathcal{S}}\approx\mathcal{S}$ on his side by running $\dot{\hat{x}}(t)=\hat{f}(\hat{x}(t))$ starting from the initial state $\hat{x}(0)$. Notice that the state dynamics may cover only a small subset of the state space, which removes the need to model $f$ on its whole domain. We use this property of dynamical systems for compressing their information and obtaining an optimal local trade-off between model complexity and prediction accuracy.

% In this paper, we take a closer look at the description of a dynamical system in the presence and absence of a model and based on that, we propose a model selection criterion with practical use in empirical inference.

% Therefore, how can this sequence be encoded into a universal computing machine such as human brain~\citep{ornstein1975experience} But where does this internal feeling come from? Let’s assume human brain is an encoding machine that converts perceptible signals into the codes realized by neural activities. It makes sense that complicated neural activities were caused by complicated objects~\citep{barlow1961possible}. The objects can be physical concrete objects or abstract mathematical creatures~\citep{copeland1970children} such as dynamical systems that model the laws of nature~\citep{newton1833philosophiae}.To our eyes, a dynamical system is nothing more than a temporal sequence of observations. Therefore, how can this sequence be encoded into a universal computing machine such as human brain~\citep{ornstein1975experience}. In this paper, we take a closer look at the description of a dynamical system in the presence and absence of a model and based on that, we propose a model selection criterion with practical use in empirical inference.

The underlying questions of this example are highly relevant also for 
%engineered 
artificial intelligent systems.  Imagine autonomous vehicles flying or driving in a formation \citep{alam2015heavy}, or multiple robots coordinating their actions \citep{rubenstein2014programmable}.  These systems need to know of each other; that is, agents need to transmit dynamics information between each other.
%, more precisely, of the other agents' state.  
An intelligent agents, however, will use its resources wisely and thus communicate only when and what is necessary.  In this scenario, better encoding of dynamical systems means reduced communication, lower bandwidth requirements, and thus reduced cost.  Likewise, intelligent agents may store various internal models for the purpose of simulation, prediction, or control \citep{camacho2013model}.  
Better representations here may mean improved performance, reduced memory requirements, and faster computation.

% \fr{Write paragraph about doing something between the two extremes - transmitting solely states and a full blown model. Instead do local models + some states to iterate from.}

{\it Contribution ---} In this paper, we propose to encode dynamical systems through local representations, which are computed to yield (locally) an optimal trade-off between model complexity and predictive performance.  The criterion automatically chooses the `right' complexity -- locally simple dynamics are represented by low-order models, while higher-order representations are automatically taken in areas with more complex dynamics. Because the representation thus adapts to the local information content of the dynamics, the proposed encoding scheme represents a novel information criterion for dynamical systems, which we call \emph{Local Information Criterion for Dynamical Systems} (\licds).  

\licds is motivated by compressing information through local representations.  Since there are theories and empirical evidence in machine learning confirming the relation between generalization and compression~\citep{vapnik2013nature, luxburg2004compression},
% \seb{@Arash: can we say something like this (was trying to recall what you once said in our discussions)?  Can you check and possibly add a reference?}\ar{yes. There are literature on this. how about this sentences? Since there are theories and empirical evidence in machine learning confirming the link between generalization and compression~\citep{vapnik2013nature, luxburg2004compression}, we hypothesized that \licds be used for model selection as well.} \seb{I'm fine with the sentence; feel free to change accordingly and remove comments},
we hypothesized \licds is also useful for model selection.  
Indeed, we show that the information criterion can be used (in addition to efficient encoding) to choose among different models learned from a given dataset. 
%In addition to efficient encoding, we thus show that the information criterion can be used to choose among different models learned from a given dataset. 
In particular, we show that it is possible to choose between different architectures of neural networks (NNs) and to compare different types of learned models (e.g., NNs versus GPs) solely with the aid of the compression score and not with test data.  We extend our empirical findings with theoretical results, which confirm the correctness of our method for certain function classes and provide insight into why \licds is a useful criterion for model selection.
Fig.~\ref{fig:cartoons} illustrate the two proposed applications of \licds: encoding and model selection.

% Previous version:
%In this paper, we introduce a new information criterion, based on the proposed local encoding scheme. We call it local information criterion for dynamical systems (\licds) and motivate it by compressing information through the local representation. The schematics are depicted in Fig.~\ref{fig:cartoons}.
%
%In addition to efficient encoding, we show empirical results that the information criterion can be used to choose among different models learned from a given dataset. 
%In particular, we show that it is possible to choose between different architectures of neural networks and compare various types of learned models solely with the aid of the compression score and not with test data, e.g. neural networks vs GPs. Additionally, we extend our empirical findings with theoretical results, which confirm the correctness of our method for certain function classes. 

\begin{figure}[t!]
	\centering
	\hspace{-0.3cm}
\subfigure[Encoding]{
\includegraphics[width=0.45\linewidth]{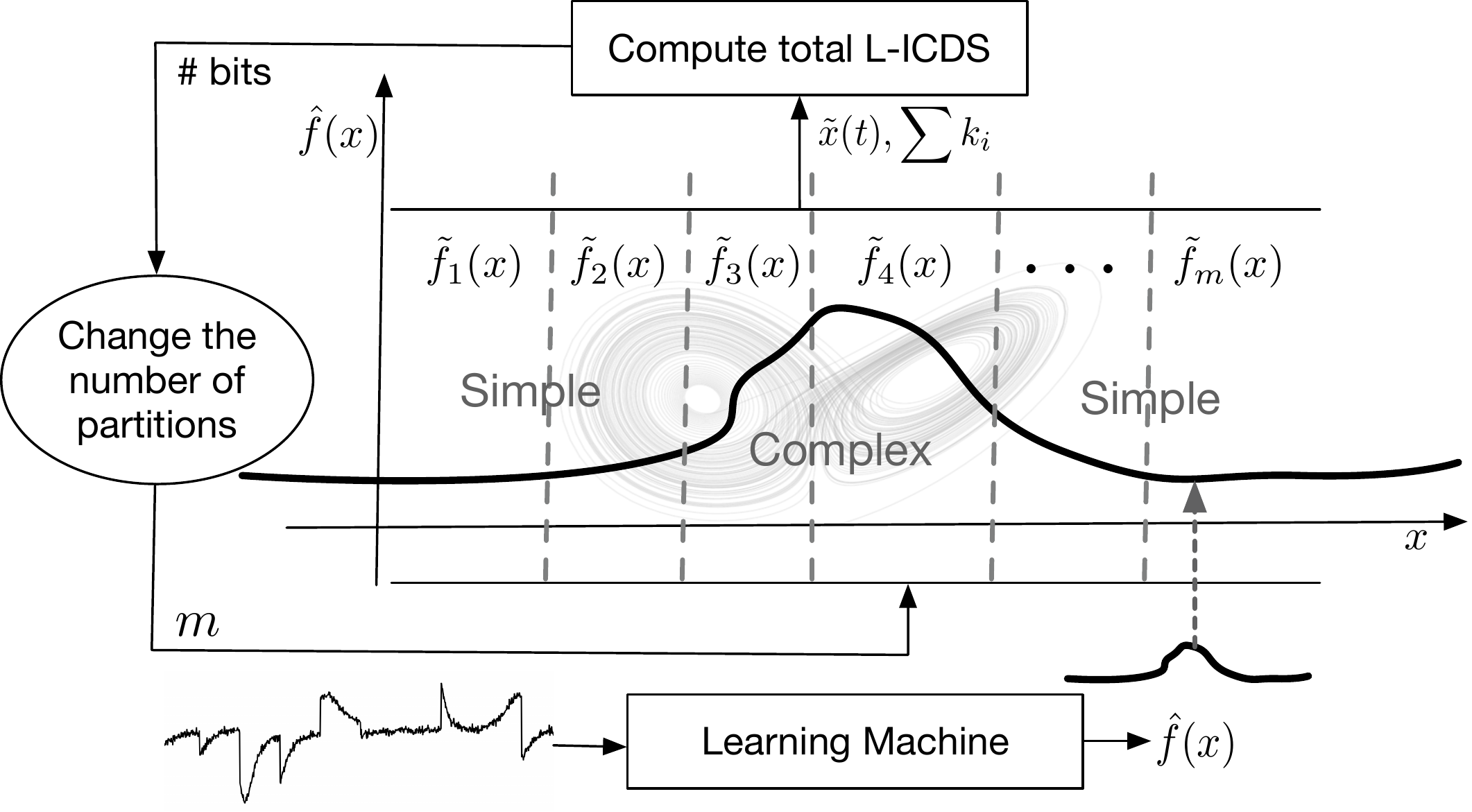}}
\subfigure[Model selection]{
\includegraphics[width=0.45\linewidth]{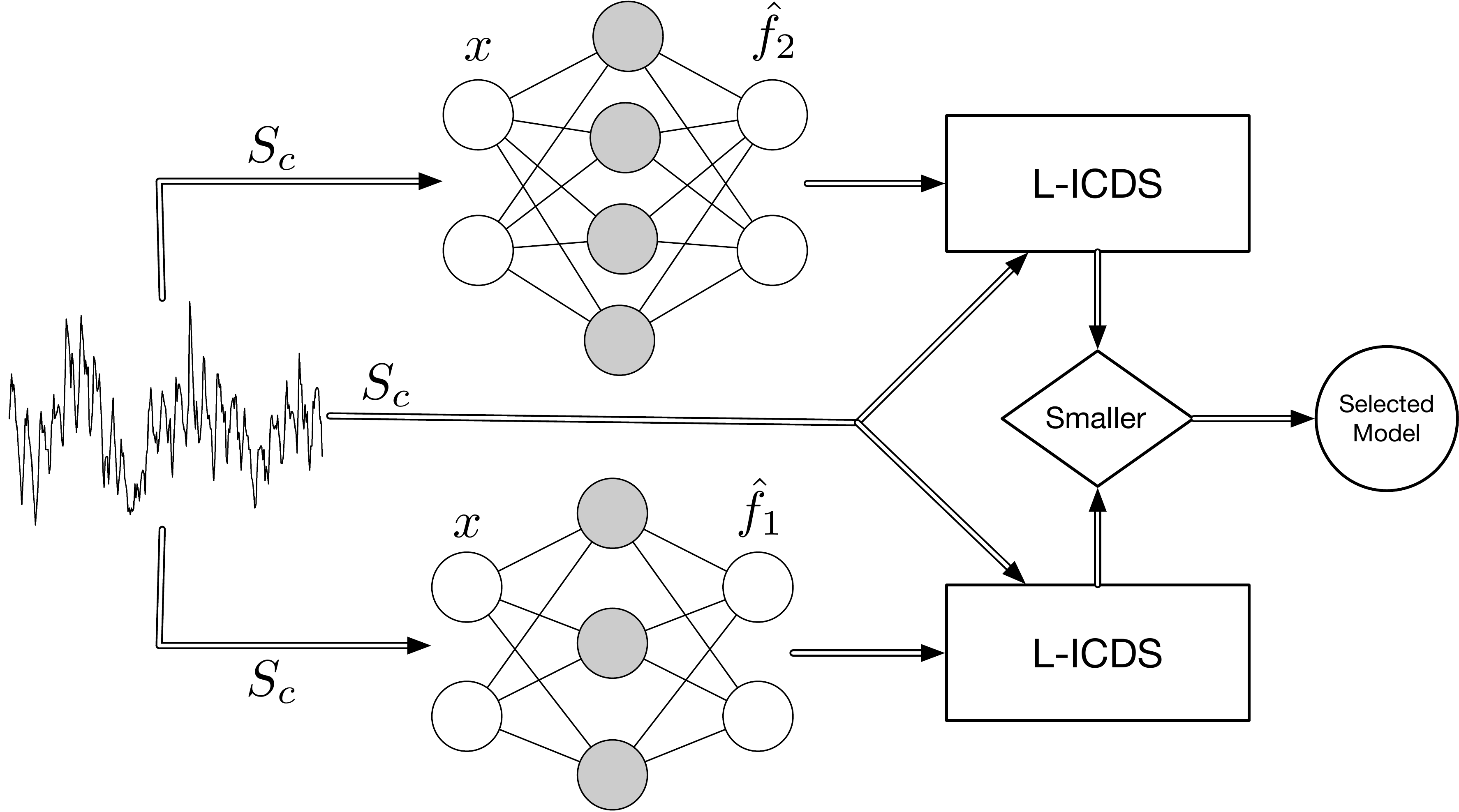}}
\caption{\small Two cartoons for two ideas presented in this paper. 
(a) Encoding a time series by first learning a model from the observations (bottom) and then locally encoding the model by \licds. (b) Using the \licds score as an information criterion to perform model selection. Two different neural network architectures are trained from a time series and \licds score is computed for each. 
%The model that achieves the smaller \licds score gives the best compression of $(S_c, \hat{f})$ and is selected based on MML principle.
}
\vspace{-0.2cm}
\label{fig:cartoons}
\end{figure}

{\it Related Work ---} 
The subject of obtaining a representation of a dynamical system from its input-output data is known as system identification \citep{ljung1999system,nelles2013nonlinear} or model learning \citep{nguyen2011model}. 
Two major approaches in system identification are gray-box and black-box approaches~\citep{ljung1999system,nelles2013nonlinear}. 
Gray-box methods learn the parameters of a known model~\citep{tulleken1993grey}, where parameters typically have a physical interpretation. However, black-box methods need to identify both the structure and parameters of the model~\citep{sjoberg1995nonlinear}. In black-box system identification, or machine learning in general, choosing the appropriate structure is usually done by investigating model performance on a left-out validation set. Information criterion is a different approach to model selection by taking into account model complexity and data explanation at the same time~\citep{yamaoka1978application}. Many information criteria have been proposed and used for supervised~\citep{fogel1991information} and unsupervised learning~\citep{mehrjou2016improved}. Despite some recent work~\citep{darmon2018information, mangan2017model} on the the information criterion approach towards dynamical systems, the field is not explored well yet. This work is proposing a compression method for dynamical systems that can be used as an information criterion and for model selection as well.  

% of t General approaches to learning/identifying dynamic models from data; we can
% point to some classic Sys ID books (e.g. Ljung), as well as to some results more
% from the ML community on model learning (there is one survey on model learning by
% Nguyen-Tuong + Jan Peters that I’m aware of, but there are probably more).

Models of dynamical systems take very different representations.  On the one end of the spectrum, there are classical parametric models such as linear transfer functions and state-space models \citep{ljung1999system}, as well as nonlinear gray-box models with known structure (e.g., from first principles) and some free parameters.  In these, the model structure is relatively rigid and information is encoded in a small number of parameters, often with some physical interpretation.  
%While being parametric likewise, 
Neural networks (NNs)~\citep{wang2016data,narendra1990identification} can also learn model structure and encode information in a large number of weight parameters without direct physical interpretation.
Fuzzy models such as Takagi-Sugeno \citep{takagi1993fuzzy} encode dynamical systems as a set of fuzzy rules or sets and associated models.
% other ref for fuzzy: \citep{angelov2004approach}
Nonparametric methods such as Gaussian process (GP) models \citep{frigola2013bayesian,doerr2017optimizing,eleftheriadis2017identification} and classical time- or frequency-domain methods \citep{wellstead1981non}
%\citep[Cha.~]{ljung1999system} 
represent dynamical system information essentially in a dataset (in time or frequency domain).  Herein, we propose to encode dynamical systems in local models whose complexity is adapted to the data stream.  Our encoding thus provides a middle ground between encoding in a dataset and a single (global) parametric model.
%The encoding that we propose herein strikes a balance between the two ends of 
%
%Black-box learning of dynamical systems has been studied from different approaches such as: GP~\citep{archer2015black}, Neural networks~\citep{wang2016data}, Fuzzy models~\citep{angelov2004approach}, and etc. 

The benefits of local modeling approaches for dynamical systems have long been recognized \citep{atkeson1997locally,nelles1996basis,TiMeViSc16}.  These include, in particular, the abilities to learn fast and incrementally from a continuous stream of (possibly large) data, while allowing for non-stationary distributions \citep{TiMeViSc16}.
% meier2014incremental
This is critical in real-time learning such as robot control \citep{ScAt10}.  While in these works the complexity of the local model must be chosen a priori (most often, locally linear models are used), our method allows for determining the optimal model complexity, which is adapted to the data.
%
%They usually take a global approach to learning the sysyem and   Discuss some local learning approaches in the context of dynamic systems; for
%example, the work by Stefan Schaal (a series of papers on locally weighted learning,
%locally weighted receptive field, etc.).  They argue about why it can be beneficial
%to do local learning in the context of dynamic systems + control.  What
%distinguishes us from them is that we reason about the complexity of the local model
%in a principled way, while they typically fix the local model order (often linear).

The proposed encoding scheme for dynamical systems was first considered in \citep{solowjowCDCsubm_arxiv},
% \footnote{Reference omitted for anonymous review. The paper is currently under review at a conference and, if accepted, will be properly cited for the potential final version of this paper.}, 
but in a different context from the one herein.  While in \citep{solowjowCDCsubm_arxiv},
the true dynamics are assumed \emph{known} and \licds is used to efficiently communicate state information between agents in a networked system, we consider encoding of dynamics models \emph{learned} from data.
%a setting where the true dynamics are unknown and to be learned from data.  
Moreover, the proposal of \licds for model selection (\sect \ref{subsec:info_modelselection}), and all theoretical (Theorems 1-3) and empirical results (\sect \ref{subsec:experiments_local_encoding} and \ref{subsec:experiment_modelselection}) are novel contributions.

\section{Proposed Local Encoding}
\label{sec:proposed_local_encodin_method}
In this section, we explain our proposed encoding scheme for dynamical systems 
\begin{equation}
\dot{x}(t) = f(x(t)), x(0) = x_0, t\in [0, T].
\label{eq:dynamical-system}
\end{equation}
% \fr{Make problem precise. Introduce $L$ on an abstract level as an approximator and argue that argmax(L) gives good encoding, while the value of $L$ can be used for model selection. }
%
% \fr{Rework UTM arguments. Collect UTM and AC arguments in one place, right now they are scattered with some repitition. See Sec. 4 }
We present our idea based on the concepts of algorithmic complexity~\citep{wallace2005statistical}, Universal Turing Machines (UTM)~\citep{turing1937computable}, and minimum message length (MML)~\citep{wallace1999minimum}. UTM is a programmable machine that receives a message as input and produces the desired output. Minimum length of the input message can be seen as the complexity of the output and is called algorithmic complexity (AC). The MML principle chooses a model for the observed data where the joint AC of the tuple (model, data) is minimized. The detailed prerequisite definitions are delegated to the supplementary material.
% \seb{I find the previous paragraph is pretty abstract, but I guess it is fine, because we want to make connections to these principles?!  However, if there are technical definitions that are not common or clear from context, these should be introduced here, not in the supplementary material.  The main paper should be understandable without the supplementary material.}\ar{I think the definitions are enough to understand the text.}

{\it General notion---} Our aim is to construct a brief and efficient \emph{explanation} for the observed data $x(t)$ from the model. The explanation is a message consisting of two parts. The first part encodes some general \emph{assertion} (theory) $\hat{\theta}$ about the source of data and the second part is the explanation for the data were the assertion is correct~\citep{wallace1999minimum}. Throughout this paper, we assume the data takes finite discrete values with certain precision. Hence, each data example can be encoded to a finite sequence of `0s' and `1s'. This is a reasonable assumption because, in practice, values are usually stored in a quantized way on digital computers, and we shall consider a finite horizon hereafter.
% \seb{I would rewrite: This is a reasonable assumption because, in practice, values are usually stored in a quantized way on digital computers, and we shall consider a finite horizon hereafter.}
% The first part of the message encodes a theory or belief $\hat{\theta}$ about the data generating process. The second part of the message encodes the data given that the theory $\hat{\theta}$ is correct. 

% \seb{Given that we use this Alice-story also later, I would include it already here and write: {\it Alice encoding information of the dynamical system---}}\ar{changed}
{\it Alice's encoding of a dynamical system---} Assume Alice is given a long sequence of observations $S=[x(1), x(2),\ldots,x(N)]$ to be transmitted to Bob over the phone. Alice thinks of a message $I=I_1.I_2$ consisting of two parts. The first part $I_1$ encodes her belief about the dynamical system   
% \seb{Please double-check that notaiton is consistent with intro; for example:  $\mathcal{S}=[x(1), x(2), \ldots]$, but also check other notation.}
% The sequence is long enough that communicating the sequence by saying its members one after another is not such a good idea.

% \seb{I would make the algorithms either t or b, not H. Also, the algorithms need to be refered to from the text, and should be place somewhere near where they are mentioned in the text.}\ar{I had to put them side by side to take less space. Setting it to t disorganized it. I could not find an easy solution}

\begin{adjustbox}{minipage=0.55\textwidth, left}
%\vspace{-1cm}
\begin{algorithm}[H]
\algsetup{linenosize=\small}
\small
\caption{\mieds: Computing the efficient code}
  \begin{algorithmic}[1]
   \renewcommand{\algorithmicrequire}{\textbf{Input:}}
   \renewcommand{\algorithmicensure}{\textbf{Output:}}
   \REQUIRE Dynamical function $f$, initial state $x_0 = \tilde{x}(0)$, \\
   global time horizon $T_{\rm global}$, maximum number of partitions $m_{\rm max}$, maximum number of expansion terms for each local model $k_{\rm max}$
   \ENSURE  Approximated states $\tilde{x}$, optimal total cost $L_{\rm total}^*$, optimal total complexity $k_{\rm total}^*$, optimal number of partitions $m^*$
   \STATE $x_{\rm exact}(t)\leftarrow \dot{x}=f(x(t))$: Observations
   \STATE $L_{\rm total}^*\leftarrow 0$: Optimal total cost
   \STATE $k^* \leftarrow 0$: Optimal total complexity
   \STATE $m^* \leftarrow 1$: Optimal number of partitions
   \STATE $T_{\rm local}\leftarrow T_{\rm global}/m$: Local time horizon
    \FOR{$m \in \{1,\dots,m_{\rm max}\}$}
    	\STATE Reset $L_{\rm total}$ to $0$
        \STATE Reset $k_{\rm total}$ to $0$
        \STATE Reset $\tilde{x}$ to $[x_0]$
      \FOR{$i \in \{1,\dots,m\}$}
      	\STATE $x_0\leftarrow x_{\rm exact}((i-1)\times T_{\rm local}+1))$ 
      	\STATE $[\tilde{x}_i, L_i^*,k_i^*] = {\rm LMS}(f,i,T_{\rm local},\lambda,x_0,k_{\rm max})$
        \STATE $\tilde{x}\leftarrow [\tilde{x}, \tilde{x}_i]$
        \STATE $L_{\rm total}\leftarrow L_{\rm total} + L^*_i$
        \STATE $k_{\rm total}\leftarrow k_{\rm total} + k^*_i$
      \ENDFOR
      \IF{$m$ = 1}
        \STATE $L_{\rm total}^* \leftarrow L_{\rm total}$
        \STATE $k_{\rm total}^* \leftarrow k_{\rm total}$
      \ELSIF{$m>1$ and $L_{\rm total}<L_{\rm total}^*$ } 
      \STATE $L_{\rm total}^* \leftarrow L_{\rm total}$
      \STATE $k_{\rm total}^* \leftarrow k_{\rm total}$
      \STATE $m^* \leftarrow m$
      \ENDIF          
    \ENDFOR
    \RETURN $[\tilde{x}, L_{\rm total}^*, k_{\rm total}^*, m^*]$
  \end{algorithmic}
  \label{alg:mieds}
\end{algorithm}
\end{adjustbox}
\begin{adjustbox}{minipage=0.4\textwidth, right}
\vspace{-15.25cm}
 \begin{algorithm}[H]
 \algsetup{linenosize=\small}
\small
 \caption{LMS: \small Local Model Selection}
 \begin{algorithmic}[1]
 \renewcommand{\algorithmicrequire}{\textbf{Input:}}
 \renewcommand{\algorithmicensure}{\textbf{Output:}}
 \REQUIRE Dynamical function $f$, index of local window $i$, local time horizon $T_{\rm local}$, relative weight $\lambda$, initial state of local time horizon $ x_0$, and the maximum allowed complexity $k_{\rm max}$.
 \ENSURE Local approximate of state trajectory $\tilde{x}$, optimal local cost $L^*_{\rm local}$, optimal local complexity $k^*_{\rm local}$
 \STATE $x(t)\leftarrow \dot{x}=f(x(t));\, x(0)=x_0$:  Observations
 \STATE $L_{\rm local}^*\leftarrow 0$: Optimal local cost
 \STATE $k_{\rm local}^*\leftarrow 0$: Optimal local complexity
 \STATE $t_{\rm start} \leftarrow (i-1) \times T_{\rm local}$
 \STATE $t_{\rm stop} \leftarrow i \times T_{\rm local}$
 \FOR {$k = 1, \ldots , k_{\rm max}$}
 \STATE $\tilde{x}(t)\leftarrow \dot{x}=\tilde{f}(x(t));\, x(0)=x_0$
 \STATE $L_k =  \lambda k + \int_{t_ {\rm start}}^{t_{\rm stop}} \| \tilde{x}(t) - x(t) \|{\it d}t$
 
 	\IF{$k=1$}
    \STATE $L_{\rm local}^*\leftarrow L_k$
 	\ELSIF{$k>1$ and $L_{k}<L^*_{\rm local}$}
    	\STATE $L^*_{\rm local} \leftarrow L_{k}$
    	\STATE $k^*_{\rm local} \leftarrow k$
    \ENDIF
 \ENDFOR
 \RETURN $[\tilde{x}, L^*_{\rm local}, k^*_{\rm local}]$
 \end{algorithmic} 
 \label{alg:partition}
 \end{algorithm}
\end{adjustbox}
that has generated the sequence, and the second part $I_2$ encodes the unexplained portion of the data by the assumed dynamical system. In this setting, Bob has a UTM that decodes $I$ and reconstructs the original sequence.  The first part of the message $I_1$ teaches Bob Alice's belief $f$ about the source dynamical system, and the second part $I_2$ teaches Bob how to recover the observations given this dynamical system.  Assume that Alice and Bob have agreed on a finite set of functions $\mathcal{H}$ from which the dynamics $f$ is chosen. Therefore, the first part of the message takes $\log |\mathcal{H}|$ bits to choose one member of this set. The second part of the message $I_2$ encodes the initial point $x(0)$, from which the dynamical system starts evolving. 

% \seb{I'm not following this point.  Before, you said that $I_2$ is the inaccuracy of the model.  To me, encoding model inaccuracy is different from encoding the initial state.  Or do you mean this in the sense that the state must be reset due to deviation in the previous predictions?  If yes, this might make sense, but needs to be explained.}
% \ar{This is the global approach. So I2 is the initial point. We have not talked about encoding inaccuracies.} \seb{Please check your text above.  You talk about inaccuracies: and the second part $I_2$ teaches Bob how to recover the observations given this possibly inaccurate dynamical system}\ar{I see. You're right. I eliminated "inaccurate"}
Again, we assume that state values are discrete, finite and chosen from alphabet set $\mathcal{X}$. 
This assumption is valid by assuming bounded value and finite precision for states. This requires $\log |\mathcal{X}|$ bits to encode the initial point. 
In total, the number of bits required to encode the sequence of observations can be seen as an Information Criterion for Dynamical Systems (\icds). For a deterministic dynamical system, having $(f, x(0))$ suffices to recover $x(t)$ for all $t>0$ (within the assumed precision). 
% \seb{That's a bit sketchy, because we have only finite precision in $x(0)$ and $f$, I don't think we can claim that we can perfectly reconstruct the original $x(t)$.  For furture work, we might want to consider using a difference equation instead of an ODE, but let's not change this now.  To fix this to some extend, I suggest to add here in parentheses: (within the assumed precision for $x$)} 
Therefore, \icds number of bits is sufficient information to recover the sequence $S$.

{\it Can Alice do better?---} The states of a dynamical system move along a certain trajectory in the state space depending on $f$ and $x(0)$. Therefore, we do not need to encode $f$ for its entire input space. 
% For example, if we observe the states of a dynamical system only around its stable equilibrium point, we can make sure that only local knowledge of $f$ is sufficient to predict the sequence of future states. 
If $f$ takes a simple shape around the working point, we can save many bits by encoding $f$ locally rather than globally. This idea results in Local Information Criterion for Dynamical Systems (\licds). 
% The global approach to encoding dynamical systems proposes that the input tape to the UTM be $I=I_1.I_2$ where $I_1$ encodes $f$ and $I_2$ encodes the initial point for deterministic systems or initial point along with the error of the predicted values by the model run from the initial point for stochastic systems. 
Assume the state space is \emph{adaptively} partitioned into $m$ pieces along the state trajectory and the complexity of the system within each partition is also adaptively chosen.
% \seb{I think it is important to clarify that this partitioning will later be computed and is adaptive; otherwise, this sounds like some sort of a gridding of the state space, which is uncool.} 
The input tape of the UTM is formed as a concatenation of several messages (instead of two as before), \ie as $I=I_1.I_2.\ldots.I_m$, where each tuple $(I_{2j-1}, I_{2j}), j\in{1,2,\ldots}$ corresponds to the local partition $j$ in the state space of the dynamical system of Eq.~\ref{eq:dynamical-system}.
In each tuple, $I_{2j-1}$ reprograms the UTM into the simulator of $j^{\rm th}$ local approximation to $f$ and decodes $I_{2j}$ to its corresponding initial point from the \emph{Observations}; It means that the local trajectory corresponding to each local model starts from a point belonging to the \emph{correct} trajectory to prevent propagating error from one local model to the next one.
Formally speaking, we look for a local representation of a function based on a finite set of basis functions 
\begin{equation}
\dot{x}(t) = f_{x^{*j}}( x^j(t) ) \quad \diffeo \quad
\dot{\tilde{x}}^j(t) = \tilde{f}^j( \tilde{x}^j(t) ) = \sum\nolimits_{i=0}^{k_{\rm max}} \alpha_i^j \phi_i^j(\tilde{x}^j,x^{*j}),
\label{eq:expansion}
\end{equation}
where $\tilde{f}^j$ is the local approximation to $f$ around $j^\text{th}$ working point $x^{*j}$. In other words, $\tilde{f}^j$ approximates the function $f$ in its $j^\text{th}$ local partition of the state space to which $x^{*j}$ belongs. The set $\{\phi_i^j\}$ is chosen from the hypothesis space $\Phi$ with cardinality $|\Phi|=k_{\rm max}$. The set $\Phi$ is chosen rich enough such that it contains basis functions that are able to approximate $f$ arbitrarily well as $k_{\rm max}\to \infty$. Different classes of basis functions can be used, \eg Taylor expansion, Fourier series, Legendre polynomials, etc~\citep{andrews1992special}. In this paper, we use Taylor expansion to showcase our points, but the concepts are generally applicable to other expansions as well. Let us next assume the coefficients are chosen from a finite discrete set $\mathcal{A}$. The coefficients are bounded because we approximate the dynamics function by a smooth function (\eg NN with tanh nonlinearity or GP) and the derivatives are bounded. In addition, the coefficient are continuous quantities, but we again assume they are represented by finite precision (as represented on a computer). Therefore, each local message $(I_{2j-1}, I_{2j})$ requires $N_j$ bits code as follows:
$N_j = |I_j| + |I_{j+1}| = k_{\rm max} \log |\mathcal{A}| + \log |\mathcal{X}|$.
%\begin{equation}
%N_j = |I_j| + |I_{j+1}| = k_{\rm max} \log |\mathcal{A}| + \log |\mathcal{X}|.
%\label{eq:abstract_local_mml}
%\end{equation}
On the other hand, if $f$ is encoded globally, we have $N_g = |I_1| + |I_2| = k_g \log |\mathcal{A}| + \log |\mathcal{X}|$ that encodes $f$ on its whole input domain. The idea of this section is that in many practical dynamical systems, $k_g$ needs to be much larger than $k_{\rm max}$ to give a good approximation to $f$ on its whole domain, which may result in $N_g>\sum_j N_j$ (see Fig.~\ref{fig:tanh_pendulum}). 

\subsection{Practical Algorithm}
\label{subsec:encoding}
In this section, we present a practical algorithm to implement the above-mentioned idea of encoding (the exposition of this subsection follows \citep{solowjowCDCsubm_arxiv}).
Taylor expansion is used as the method for local approximation to dynamics function as Eq.~\ref{eq:expansion}. \licds does not differentiate between whether the model is known ($f$) or is learned ($\hat{f}$) and considers both as the function to be locally approximated.  In this section, we simply write $f$ to refer to either one of them. The difference will however matter for model selection in \sect \ref{subsec:info_modelselection}.
% However, in model selection applications (sec.\ref{subsec:info_modelselection}), \licds is used to choose among a pool of learned models $\hat{f}$ one that matches better with the true function $f$ and generalizes better with no access to the underlying dynamics. for the time being, let's call the dynamics function $f$ either it is known or learned. 

% We propose the following cost function as a proxy to minimize the message length that describe $f$ around working point %% of the local approximation $f^j$. of the local system, we minimize the cost function
% \begin{equation}
% \min\limits_{k \in [1,\dots,k_{\rm max}]} \lambda k + \int_0^T \| \tilde{x}(t) - x(t) \|\Sh_{\Delta T}(t){\it d}t,
% \label{eq:mml_localmodelselection}
% \end{equation}

% where $\dot{\tilde{x}}(t) = \tilde{f}( \tilde{x}(t) ) = \sum_{i=0}^{k_{\rm max}} \alpha_i \phi(\tilde{x},x^*)$, $\dot{x}(t) = f(x(t))$. We consider $x$ as the exact value of states that we aim to approximate by $\hat{x}(t)$. 
% % When the true model $f$ is absent, these values are observations and $\Delta T$ is the time interval between observations.
% Hyper-parameter $\lambda$ adjusts the relative weight of complexity versus accuracy. The function $\Sh_{\Delta T}(t)$ is called Dirac comb and is defined as $\Sh_{\Delta T}(t)=\sum^{\infty}_{k=-\infty} \delta(t-k\Delta T)$. This means that the integration in Eq.~\ref{eq:mml_localmodelselection} is approximated by a sum whose summands are sampled with resolution $\Delta T$.

{\it Local time horizon---} Local approximation relies on partitioning the input space of the dynamics function $f$. Because $f$ governs a dynamcial system, partitioning $x$-space amounts to partitioning $t$-space. This means we divide the global time horizon $T_{\rm global}$ into $m$ local time horizons with length $T_{\rm local}=T_{\rm global}/m$ where $m$ is a hyper-parameter. The detailed cost function is then written as 
\begin{equation}
      k_i^* = \argmin\limits_{k_i \in [1,\dots,k_{\rm max}]} L_i(k_i)
      \quad \text{with} \quad
      L_i(k_i) = \lambda k_i + \int_{t^{\rm start}_i}^{t^{\rm stop}_i} \| \tilde{x}(t) - x(t) \|{\it d}t \label{eq:mml_detailed_modelselection}
\end{equation}
%
% Prev version:
%\begin{equation}
%  \begin{array}{ll}
%      k_i^* = \argmin\limits_{k_i \in [1,\dots,k_{\rm max}]} L_i(k_i)\\
%      L_i(k_i) = \lambda k_i + \int_{t^{\rm start}_i}^{t^{\rm stop}_i} \| \tilde{x}(t) - %x(t) \|{\it d}t \label{eq:mml_detailed_modelselection}
%  \end{array}
%\end{equation}
for each local time horizon delimited by $t^{\rm start}_i$ and $t^{\rm stop}_i$ and $t^{\rm stop}_i-t^{\rm start}_i = T_{\rm local}$ for all partitions. 
% \seb{You need to introduce all symbols. $\Sh_{\Delta T}$ is missing.  Check for others, and also for other equations (I'm not going through the math now...).}
% Note that $k_{\rm max}$ is the size of the finite set of basis functions which are allowed to be used for approximation. 
Finding the optimal local complexity is implemented by \alg~\ref{alg:partition}, which is used as a module of \licds in \alg  \ref{alg:mieds}. The total cost function is then written as $L_{\rm total}(m)=\sum_{i=1}^m L_i(k_i^*)$. The optimal number of partitions is found by 
\begin{equation}
	m^*=\argmin\limits_{m \in [1,\ldots,m_{\rm max}]} L_{\rm total}(m)
	\label{eq:optimal_num_partitions}
\end{equation}
where $m_{\rm max}$ is the maximum allowed number of partitions. The concise message of this section is that the minimum value of $L_{\rm total}$ usually occurs for $m>1$, which implies that the proposed method gives a better encoding compared with global encoding where $m=1$. Notice that $k_{\rm max}$ and $m_{\rm max}$ are hyper-parameters of the model, which are chosen by our prior idea about the complexity of the dynamics function (larger values for more complicated functions). We observed that reasonably high values for these hyper-parameters, e.g. $k_{\rm max}=8$ and $m_{\rm max}=5$ worked well for a variety of systems and benchmarks that we have considered in the paper and also in the supplementary document.
% \ar{TODO: explain more on these hyper-parameters}

{\it How to choose $\lambda$?}
The hyper-parameter $\lambda$ acts as a balancing weight between the complexity of local Taylor approximation and error in the prediction of states. It can also be interpreted from an information theoretic perspective: Assume the values of the coefficients of the Taylor expansion come from a Gaussian distribution, i.e., $a_j\sim \mathcal{N}(\mu, \sigma^2)$. In the optimal coding scheme, the number of bits required to encode the coefficients equals the Shannon entropy of the normal distribution, $\log (\sigma\sqrt{2\pi e})$. Thus, $\lambda$ is log-proportional to the variance of coefficients that is caused by the fluctuations of the dynamics functions. In the current version, we manually choose $\lambda$ such that two terms of Eq.~\ref{eq:mml_detailed_modelselection} are of the same order.

% This equation shows that $\lambda$ becomes larger as coefficients become more scattered that is the case for highly fluctuating dynamics function $f$. However, variance of the value of coefficients is not known in advance. So, we manually choose $\lambda$ according to our prior about $f$. The other idea is to set an initial value for $\lambda$ and compute the first few coefficients. Then update $\lambda$ by the variance of the available coefficients which are computed so far.

% The hyper-parameter $\lambda$ can then be factorized as $\lambda=\lambda' \log (\sigma\sqrt{2\pi e})$ where $\lambda'k\log(\sigma)$ is affected by how much scattered the values of coefficients are and term $\lambda' k \log(\sqrt{2\pi e})$ is a constant. Therefore, we can argue that $\lambda$ is proportionally related to the variance of the values of the local coefficients.
% \ar{If we can say that $\log (\sigma)$ is not too large, we can safely say that assuming $\lambda = \lambda' k$ is a valid argument.}. \ar{TODO: explain about the Gaussian assumption for the coefficients. Maybe we can argue that each coefficient is by sum of many factors when we project the function on the bases space. Then law of large numbers makes the Gaussian assumption not so bad.}

% By this sampling function, we simulate what practically happens in digital computers.

The proposed method is summarized in \alg~\ref{alg:mieds} and the schematic in Fig.~\ref{fig:cartoons}(a)
% with the message \seb{`with the message' sounds strange to me and I don't understand the connection; consider rewording and/or starting a new sentence here.} 
The general message is that a sequence of approximations $\tilde{f}^i(x) \approx f(x)|_{\Omega_{\rm local}^i}$ over the sequence of partitions $\bigcup_i {\Omega_{\rm local}^i} \subset \Omega$, leads to a better encoding, i.e., $L_{\rm total}(m>1)=\sum_{i=1}^{m} L_i(\tilde{f}_i(x)) < L_{m=1}(f)$ (see \sect 4.1 of supplementary document for an illustrative example).

\subsection{Theoretical Results}
In this section, we will prove that it is possible to control the error introduced by local approximations. We distinguish here between two objects, the states $x(.)$ and the dynamics $f(x(.))$, both as a function of time $t$. We rely heavily on the identity in Eq.~\ref{eq:dynamical-system}, which adds a lot of regularity to this problem. 
Therefore, we can derive statements of the type: if $f$ and $\hat{f}$ are close in some sense, then the state trajectories $x$ and $\hat{x}$ are close as well. And even better, the opposite is also true -- close states imply close dynamical functions.
This guarantees sufficiently accurate state prediction, while being able to reduce model complexity. Furthermore, we will elaborate later on the other direction in order to deploy \mieds as a model selection criterion.

First, we show a result that accurate local approximations imply precise state estimations. 
%For the proofs of this and the following theorems we refer the reader to the attached supplemantary material. 
The proof of this and all following theorems are given in the supplementary material. 
\begin{theo}
\label{thm:closeFunctionsToCloseX}
Consider Eq.~\ref{eq:dynamical-system} with $f$ Lipschitz-continuous on $[0,T]$.
%Assume we consider a system like in Eq.~\ref{eq:dynamical-system},
%where $f$ is a Lipschitz-continuous function on $[0,T]$. 
Furthermore, assume a Lipschitz-continuous approximation $\hat{f}$ is used to obtain state approximations $\hat{x}$. Then, for $\hat{x}(0) = x(0) = x_0$,
%we can bound the error between $x$ and $\hat{x}$ and show
\begin{equation}
    \| x(.) - \hat{x}(.) \|_{L^2([0,T])}^2 \leq T^2 \| f(x(.)) - \hat{f}(x(.)) \|_{L^2([0,T])}^2.
\end{equation}
In particular, this implies: if $\| f(x(.)) - \hat{f}(x(.)) \|_{L^2([0,T])}^2 \leq \epsilon$, then  $\| x(.) - \hat{x}(.) \|_{L^2([0,T])}^2 \leq T^2 \epsilon$.
\end{theo}
Next, we show the opposite direction: close state trajectories imply close dynamical systems.
\begin{theo}
\label{thm:closeX}
Let the assumptions of Theorem~\ref{thm:closeFunctionsToCloseX} hold, and assume 
%We consider the same setting as in Theorem 1, but additionally assume that
$\sup\nolimits_{t \in (0,T)} x(t) = M < \infty$ and $\sup\nolimits_{t \in (0,T)} x'(t) \leq \tilde{M}$. 
%In particular, states are not allowed to oscillate arbitrary fast. 
Then, we have 
\begin{equation}
\| f(x(.)) - \hat{f}(x(.)) \|_{L^1([0,T])} \leq C \| x(.) - \hat{x}(.) \|_{L^1([0,T])} .
\end{equation}
This implies: if  $\| x(.) - \hat{x}(.) \|_{L^1([0,T])} \leq \delta$ then $\| f(x(.)) - \hat{f}(x(.)) \|_{L^1([0,T])} \leq C\delta$. 
\end{theo}
\section{Experiments: Encoding Dynamical Systems}
\label{subsec:experiments_local_encoding}

\begin{figure}[t!]
	\centering
	\hspace{-0.3cm}
\subfigure[$m=1$]{
\includegraphics[width=0.23\linewidth]{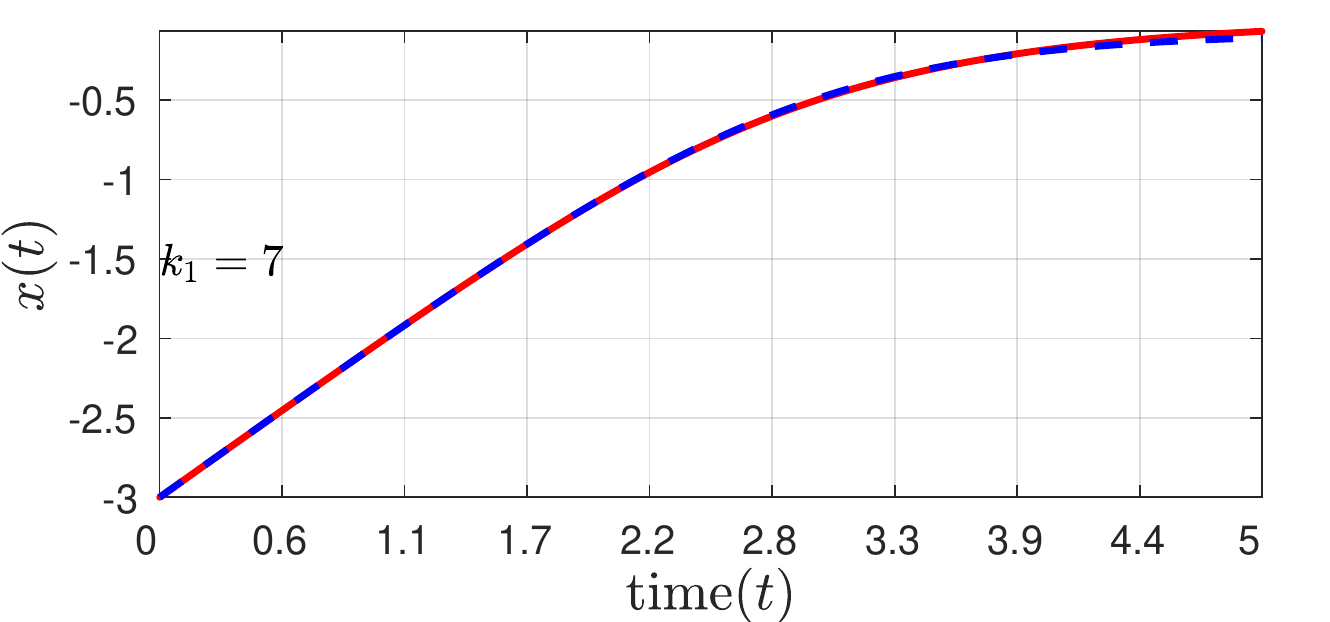}}
\subfigure[$m=2$]{
\includegraphics[width=0.23\linewidth]{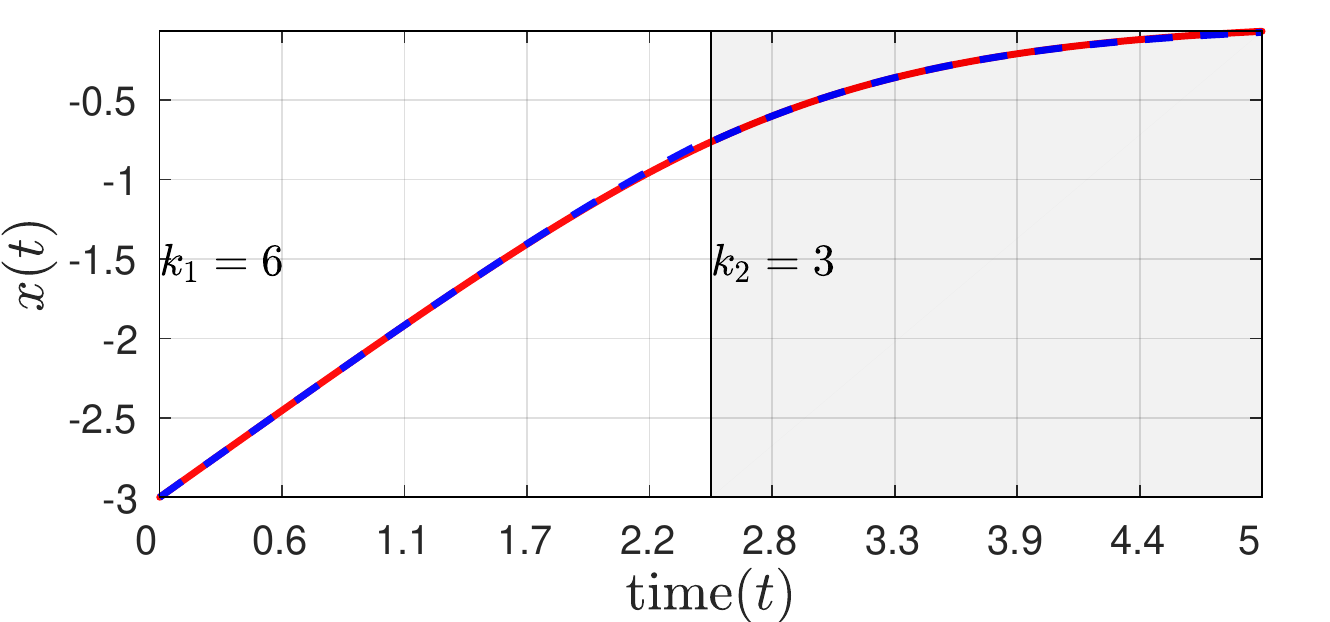}}
\subfigure[$m=4$]{
\includegraphics[width=0.23\linewidth]{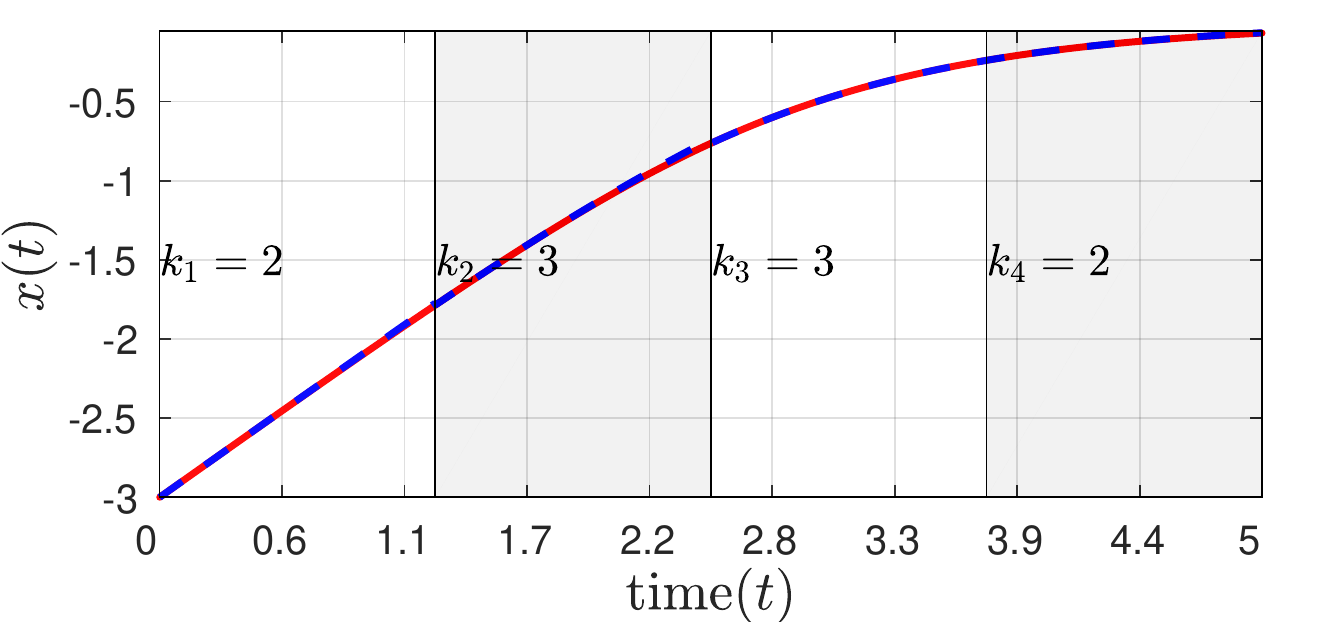}}
\subfigure[\licds score]{
\includegraphics[width=0.23\linewidth]{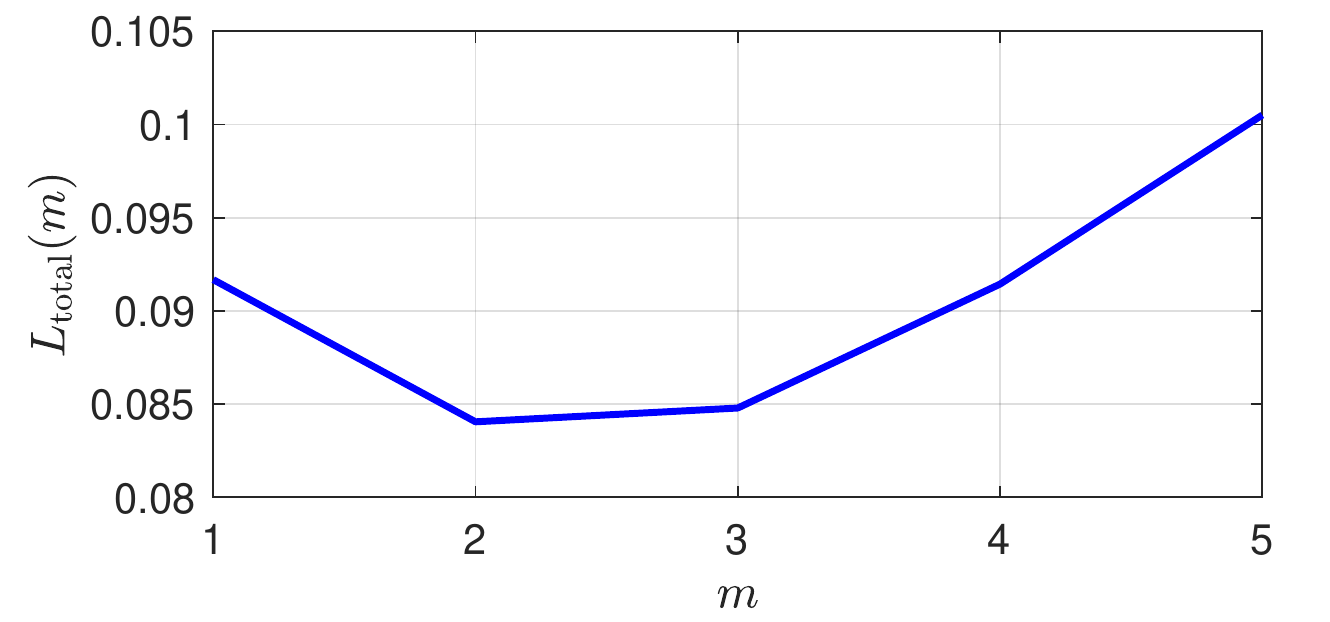}}\\
% \subfigure[4 sections]{
% \includegraphics[width=0.15\linewidth]{./figures_arxiv/nn_part_5_state_1_2d.pdf}}\\
\subfigure[$m=1$]{
\includegraphics[width=0.23\linewidth]{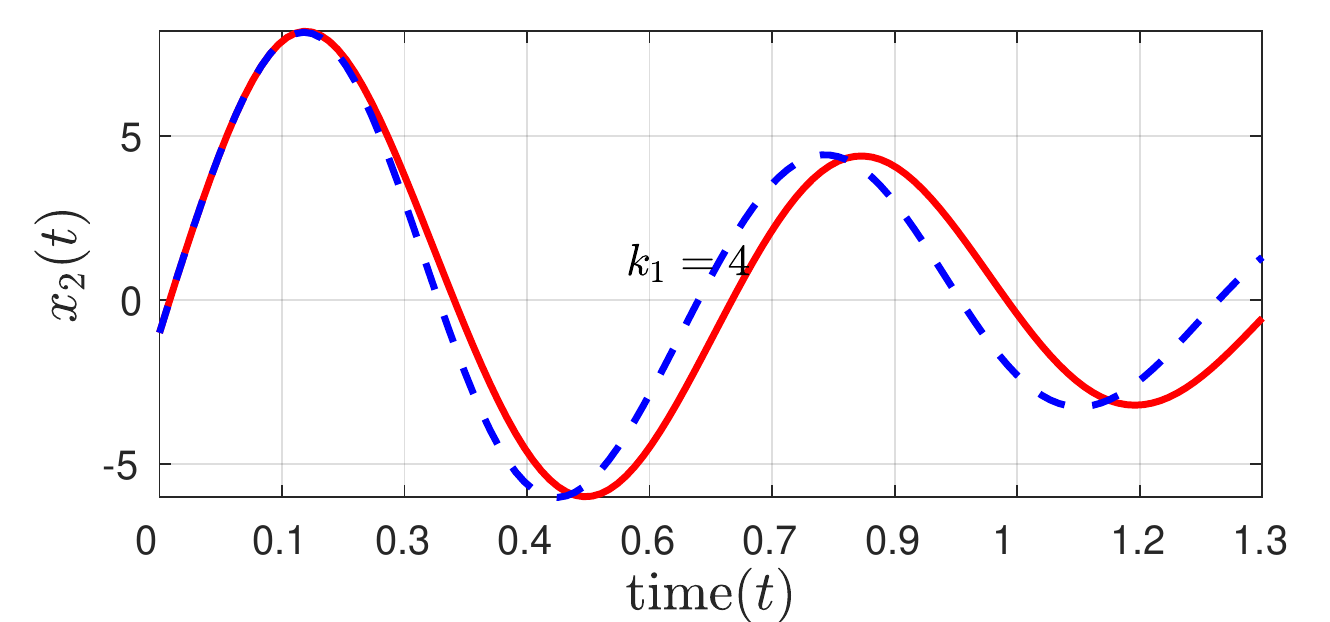}}
\subfigure[$m=3$]{
\includegraphics[width=0.23\linewidth]{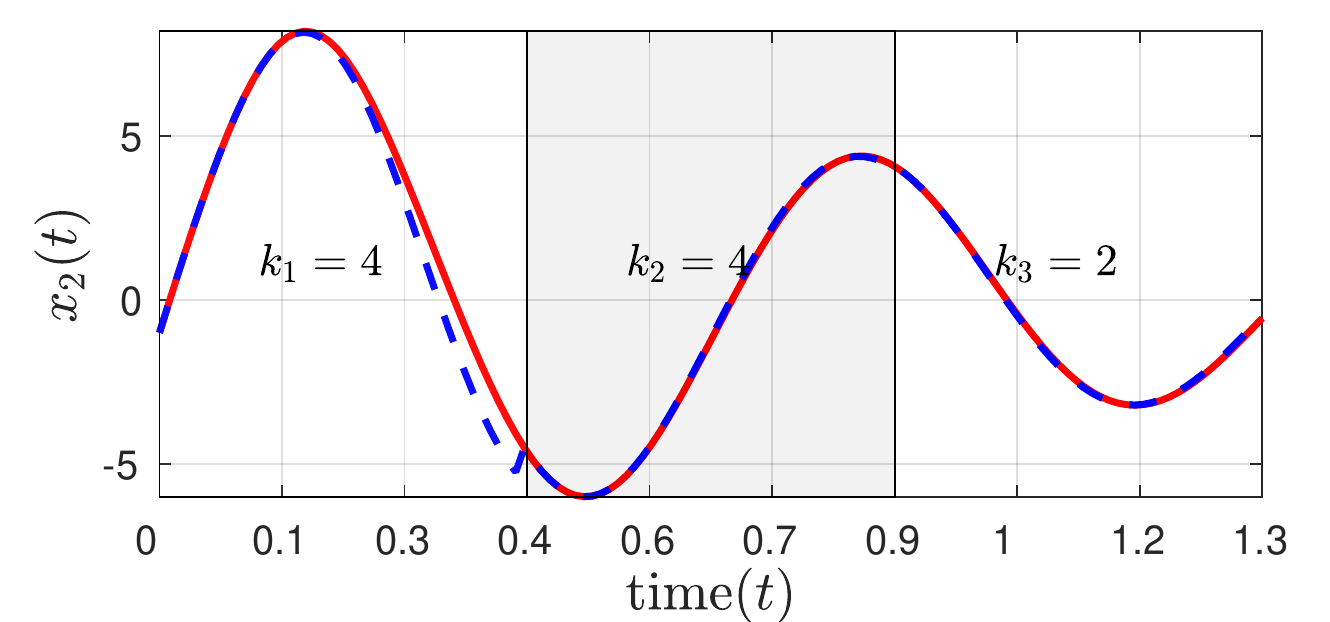}}
\subfigure[$m=4$]{
\includegraphics[width=0.23\linewidth]{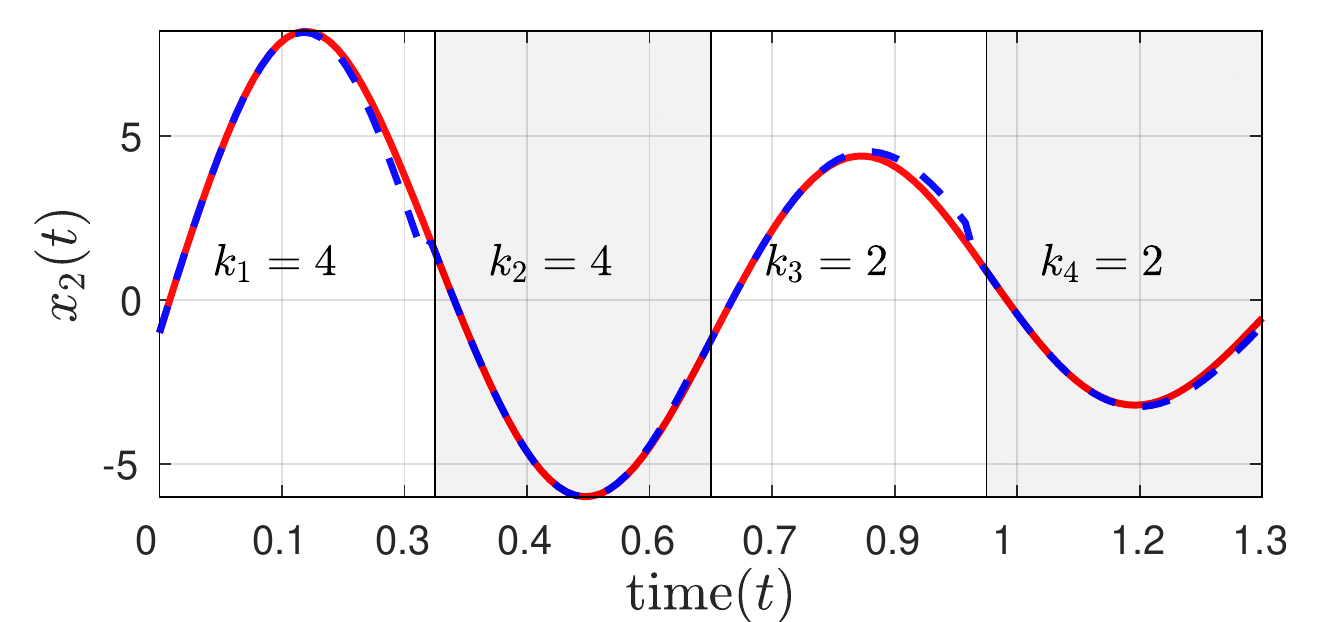}}
\subfigure[\licds score]{
\includegraphics[width=0.23\linewidth]{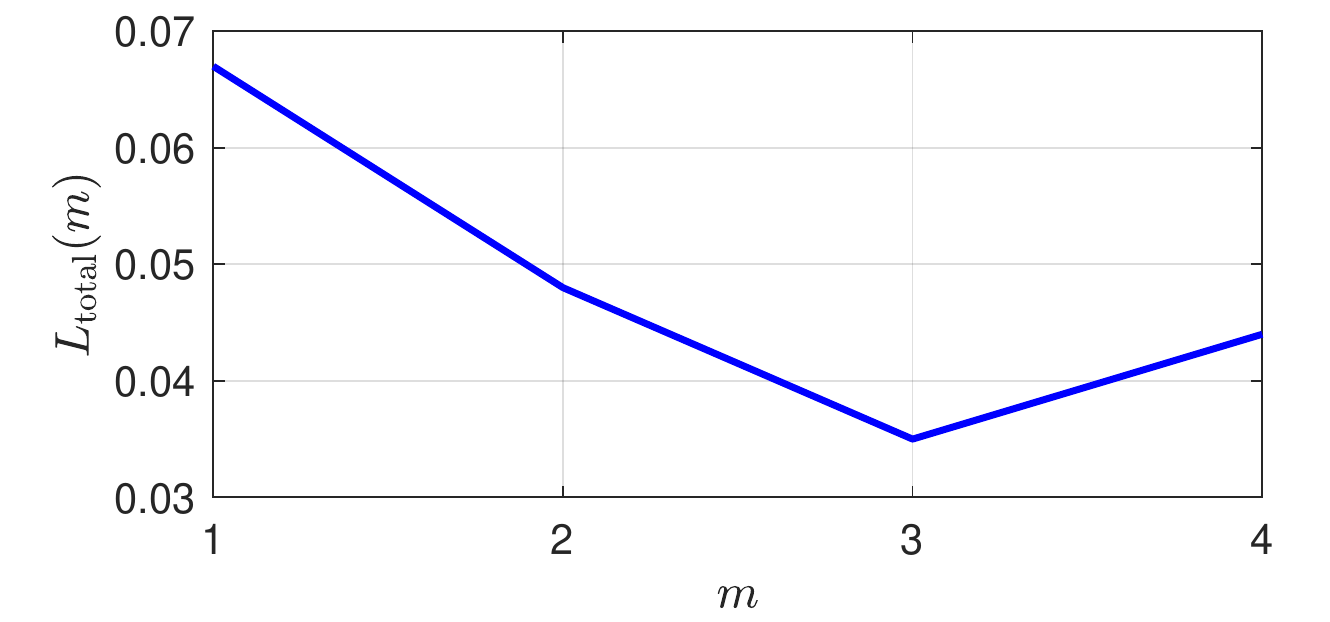}}
% \subfigure[4 sections]{
% \includegraphics[width=0.23\linewidth]{./figures_arxiv/nn_part_4_state_2_2d.pdf}}
% \subfigure[4 sections]{
% \includegraphics[width=0.15\linewidth]{./figures_arxiv/nn_part_5_state_2_2d.pdf}}
\caption{\small The accuracy of states for different number $m$ of local partitions and \licds score in the last column. (Top row) $-\tanh(x)$ system. (Bottom row) second state $x_2$ of the pendulum.}
% but \licds score is computed for both states together.}
\vspace{-0.2cm}
\label{fig:tanh_pendulum}
\end{figure}

In this section, we illustrate how the encoding scheme proposed in \sect \ref{sec:proposed_local_encodin_method} looks in practice. We elaborate in detail on the algorithm with the aid of two examples. More descriptive examples are in the supplementary material. 
We consider: 1) the one-dimensional system $\dot{x}(t)=-\tanh(x(t)) + 0.01\epsilon(t)$; and 2) a pendulum with two states $\dot{x}_1 = x_2+0.01\epsilon_1(t)$ and $\dot{x}_2 = -x_2 - 9.81{\rm sin}(x_1)+0.01\epsilon_2(t)$, where $\epsilon(t)$ is a standard white noise process. We use the Euler–Maruyama method \citep{schuss1988stochastic} to sample multiple trajectories from the systems and use those to learn the dynamics. In this example, we train a shallow NN as model $\hat{f}$ of the dynamics (details on the learning method are given in the supplementary material). The function $\hat{f}$ is now the input to the \licds algorithm, which computes a local approximation $\tilde{f}$. 
\fig \ref{fig:tanh_pendulum} depicts the functionality of the proposed method and highlights the local approximations in dependence of the number of partitions and the complexity order.
In particular, the last column of Fig.~\ref{fig:tanh_pendulum} shows that \licds prefers non-trivial solutions with $m>1$, which results in the simplest model that still gives accurate state. Similar experiments for a more sophisticated system (quadrotor) are presented in the supplementary material with similar conclusions.

\begin{figure}[tb]
	\centering
	\hspace{-0.3cm}
\subfigure[{NN = [1], $L=0.0091$}]{
\includegraphics[width=0.3\linewidth]{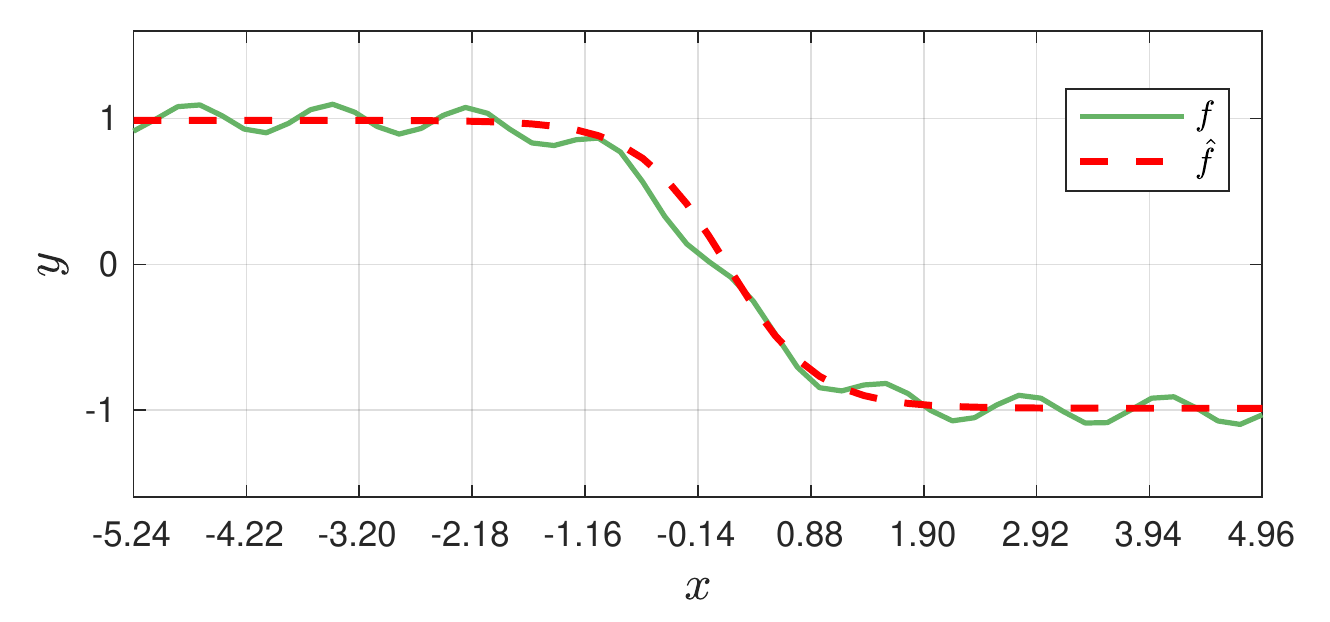}}
\subfigure[{NN = [5], $L=0.0182$}]{
\includegraphics[width=0.3\linewidth]{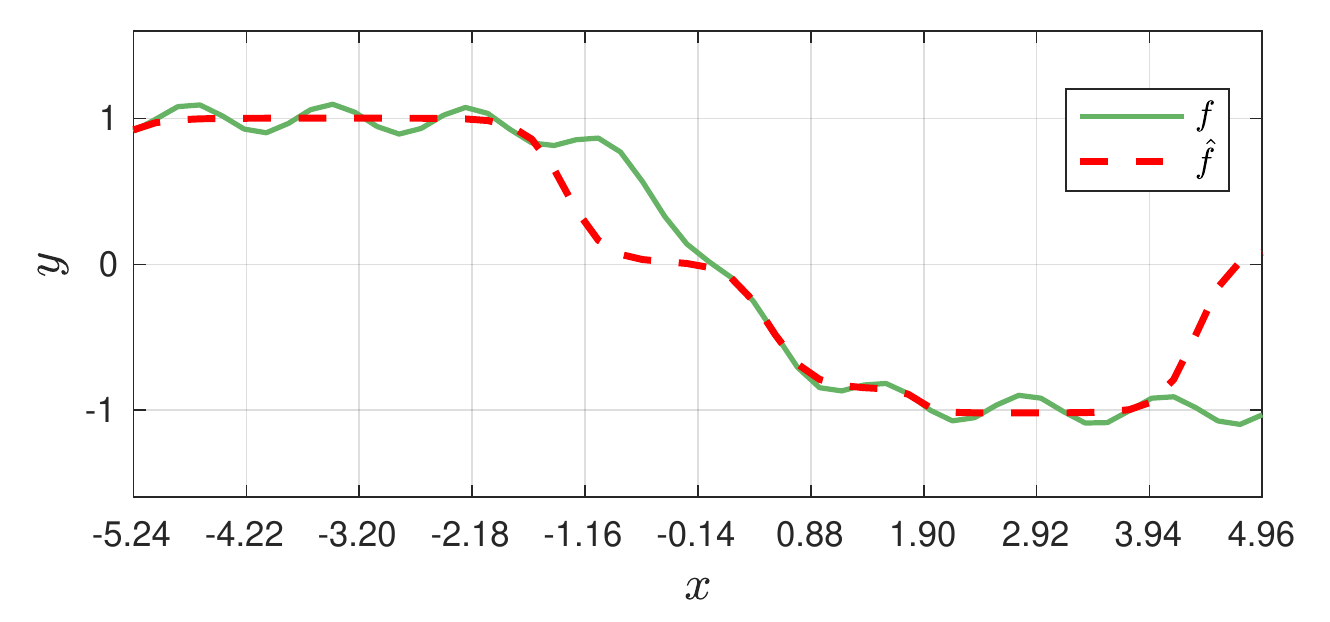}}
\subfigure[{NN = [10], $L=0.0088$}]{
\includegraphics[width=0.3\linewidth]{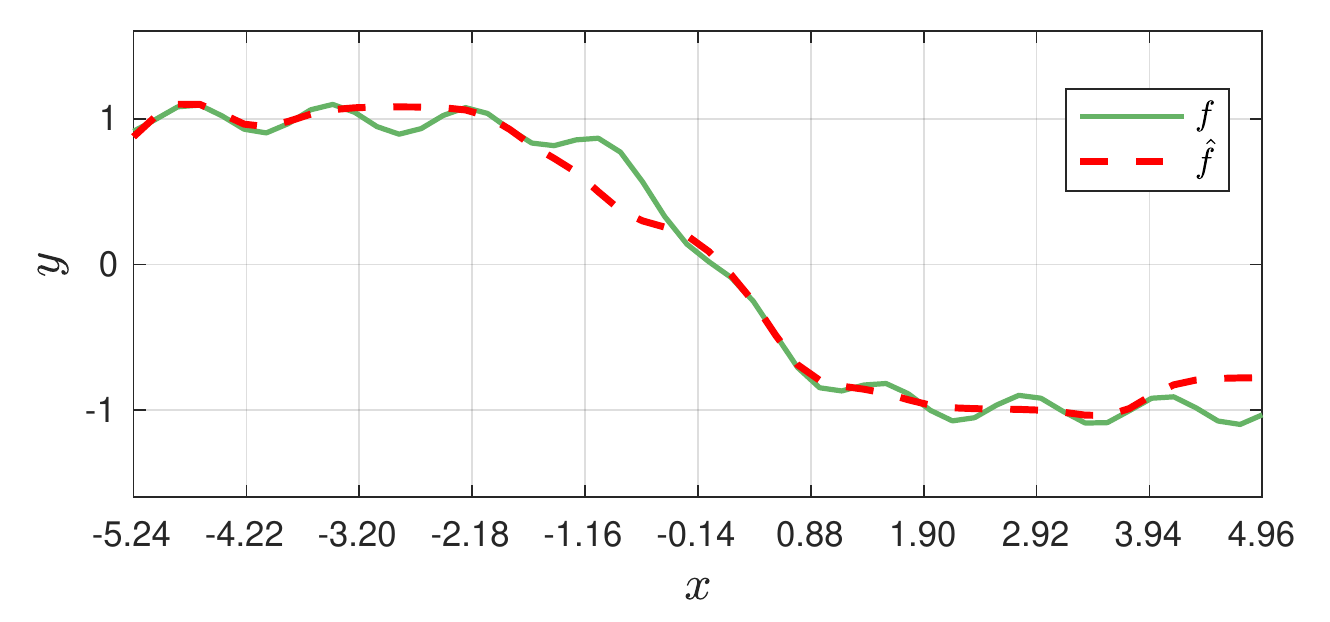}}\\
% \subfigure[{NN = [15], $L=0.0097$}]{
% \includegraphics[width=0.3\linewidth]{./figures_arxiv/f_true_learned_[15].pdf}}
% \subfigure[{NN = [50], $L=0.0101$}]{
% \includegraphics[width=0.3\linewidth]{./figures_arxiv/f_true_learned_[20].pdf}}
% \subfigure[{NN = [100], $L=0.0118$}]{
% \includegraphics[width=0.3\linewidth]{./figures_arxiv/f_true_learned_[50].pdf}}\\
\subfigure[{NN = [10, 5], $L=0.0067$}]{
\includegraphics[width=0.3\linewidth]{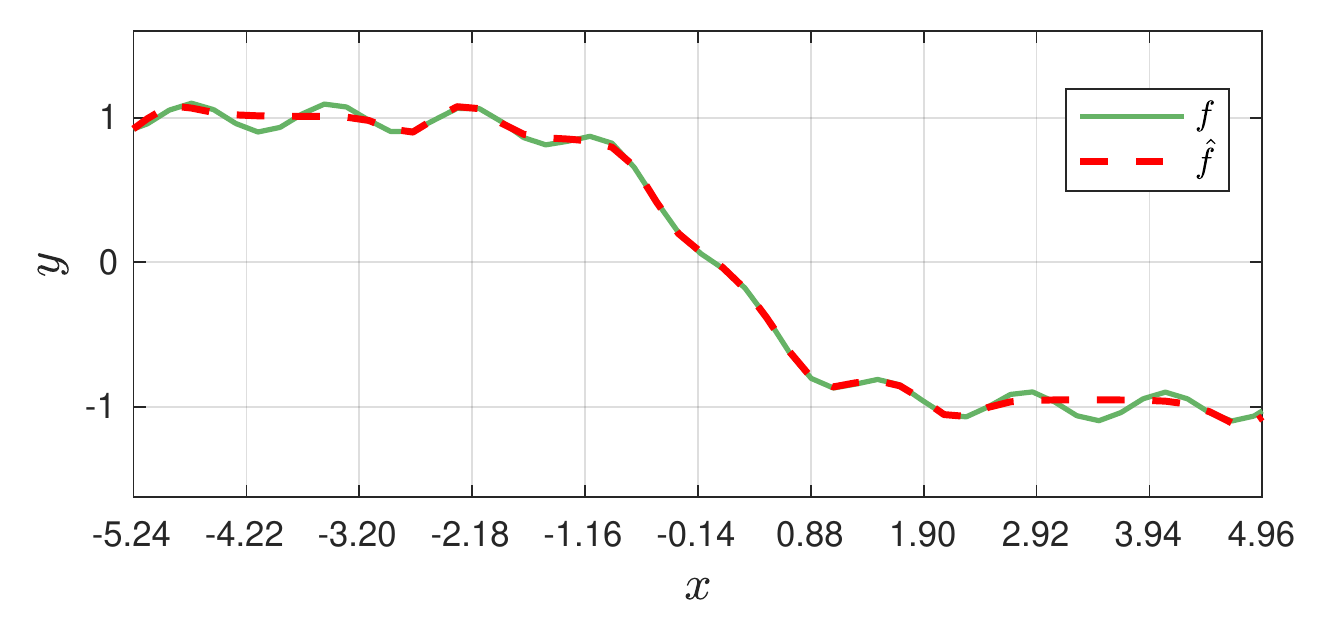}}
\subfigure[{NN = [10, 10], $L=0.0089$}]{
\includegraphics[width=0.3\linewidth]{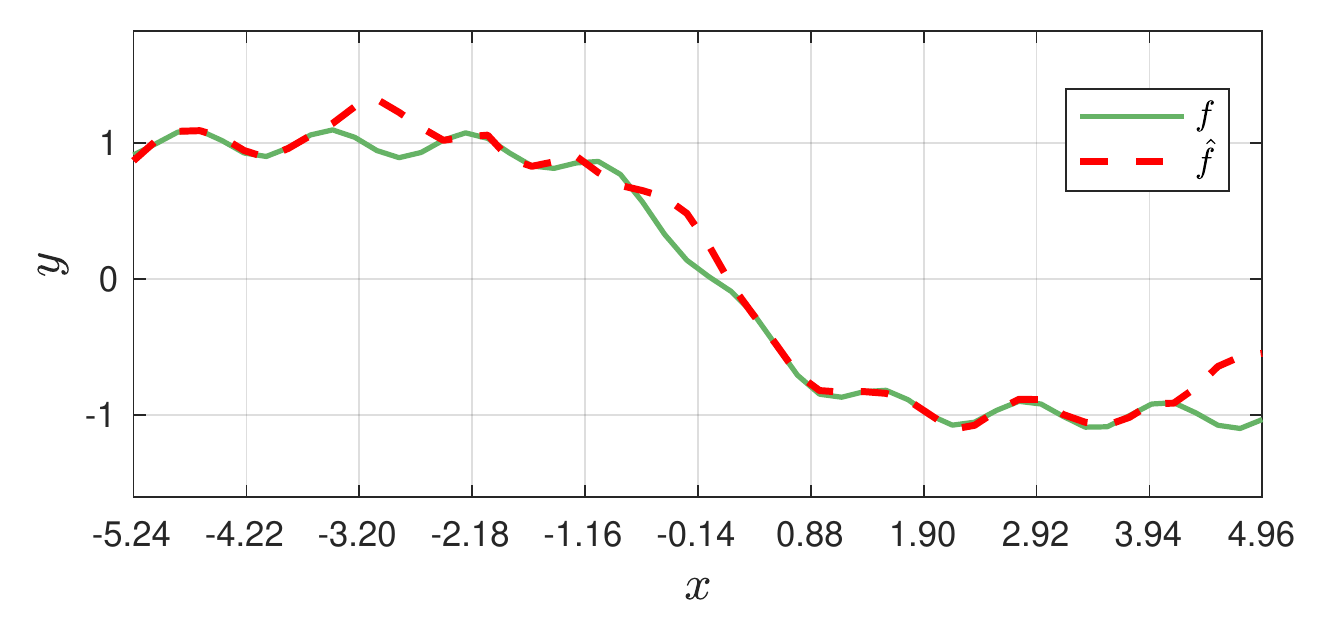}}
\subfigure[{NN = [15, 5], $L=0.0092$}]{
\includegraphics[width=0.3\linewidth]{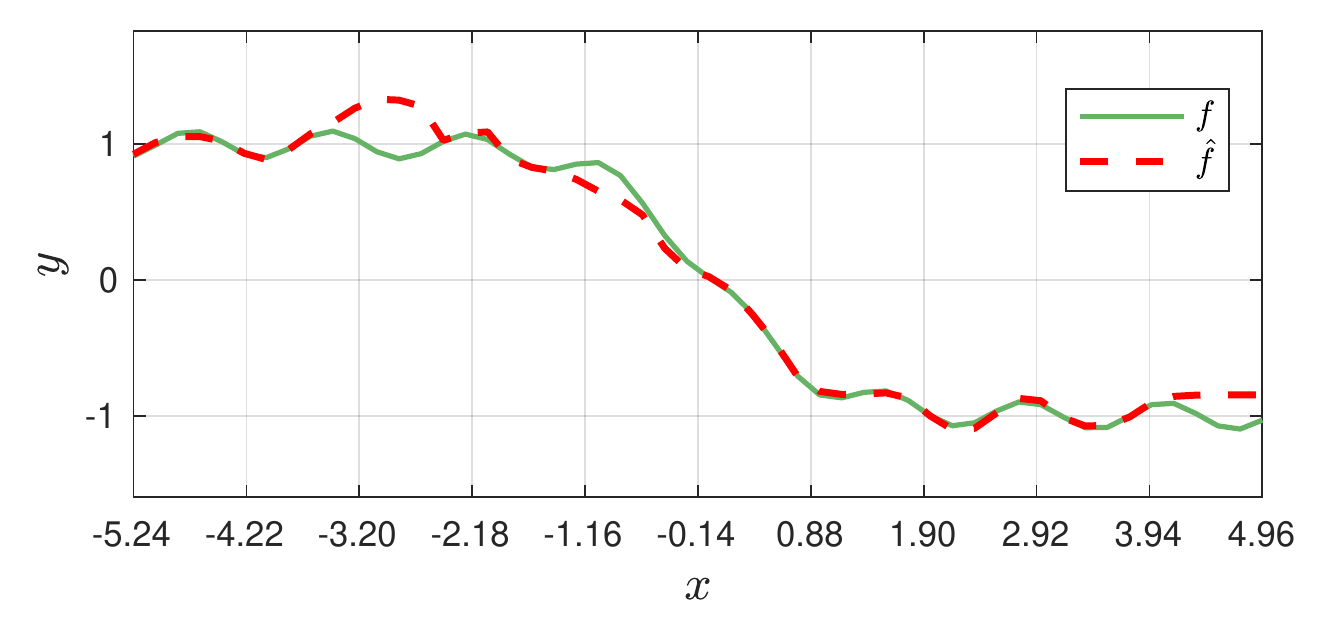}}
% \subfigure[4 sections]{
% \includegraphics[width=0.15\linewidth]{./figures_arxiv/nn_part_5_state_2_2d.pdf}}
\caption{\small Each figure shows the true function $f$ (solid) and the learned function $\hat{f}$ (dashed) by the neural network. The architecture of the used neural network is represented as the caption of each subfigure: [\#neurons of layer 1, \#neurons of layer 2, ...]. The \licds score for each model is also mentioned in each caption.}
\vspace{-0.2cm}
\label{fig:synthetic_model_selection}
\end{figure}

\section{Local Information Criterion for Model Selection}
\label{subsec:info_modelselection}
In this section, we extend our idea in order to regard \licds as a model selection criterion. 
From an abstract point of view, we can motivate our approach in terms of information compression and argue that simpler functions should be preferred when they explain data equally well. 
In addition to empirical findings, we support our claim with theoretical results, which give insight in the applicability of the proposed method.
The schematic in Fig.~\ref{fig:cartoons}(b) summarizes the key ideas of this section. 

Again, we consider the three objects $f$, $\hat{f}$, and $\tilde{f}$, which are, respectively, the true dynamical function, the output of an arbitrary learning algorithm, and the local encoding. 
Assume we are given a collection $\mathcal{S}_c$ of observed sequences $\{S_i\}$ all generated by the underlying dynamical system \eqref{eq:dynamical-system}.
%$\dot{x}(t)=f(x(t))$. 
Based on this dataset, we can deploy several different learning algorithms in order to obtain approximations $\hat{f}_1,\ldots,\hat{f}_n$. These approximations will most likely differ in their quality, which gives rise to the question, which of the learned functions should be selected.

Frequently used methods to learn dynamical functions are, for example, NNs and GPs (\cf `Related work').
%neural networks and Gaussian processes (GP).  ST: HAVE ALREADY INTRODUCED THESE
However, determining the depth of the NN and finding a suitable kernel function for the GP are non-trivial tasks. Ill-considered choices can lead to overfitting and bad performance and hence, should be discarded as soon as possible. 
For example, an over-parameterized NN may overfit to the training data and result in zero training error while being far from the correct dynamics function $f$. 

We propose a new way to compare learned functions, 
which is based on \licds and facilitates choosing among them. 
%based our encoding scheme, which should facilitate choosing among them. 
In particular, we claim that, for a certain class of functions, the function with the smallest $L$ is closest to the true dynamics. We quantify this statement in the following theorem.
\begin{theo}[Model Selection]
Let the assumptions of Theorems~\ref{thm:closeFunctionsToCloseX} and \ref{thm:closeX} hold, 
%If the previous assumptions hold, 
$L(\hat{f}_1) \leq L(\hat{f}_2)$ and 
\begin{equation}
    \| \hat{f}_1(x(.)) - \tilde{f}_1(x(.)) \|_{L^1([0,T])} \leq C \lambda k_1^*, 
    \label{eq:conditionMainThrm}
%\vspace{-1mm}
\end{equation}
%\vspace{-1.5mm}
then 
\begin{equation}
\| f(x(.)) - \hat{f}_1(\hat{x}_1(.)) \|_{L^1([0,T])} \leq E + \xi \| f(x(.)) - \hat{f}_2(\hat{x}_2(.)) \|_{L^2([0,T])},    
\end{equation}
where $E, \xi, C \geq 0$ are constants, which depend on certain properties of the dynamical systems.
\end{theo}
\begin{rem}
\label{remark:aboutTheorem}
The proposed algorithm works well for certain types of systems, which we confirm with empirical results.
However, the theorem does not guarantee that is works for all systems; clearly, for large $E$ and $\xi$, the above theorem is not very meaningful.
%we do not claim that it works for all kinds of dynamical systems and clearly
%for huge constants $E$ and $\xi$ the above theorem is not very meaningful. 
%
The condition in Eq.~\ref{eq:conditionMainThrm} is an interesting starting point to investigate the suitable class of functions, for which the theorem yields a meaningful bound.
%Condition Eq.~\ref{eq:conditionMainThrm} might be considered as technical detail and should hold on an intuitive level for some values of $\lambda$, but can also be seen as a starting point to investigate the properties and applicability of this algorithm more thoroughly in future work.
%
Finally, since Theorem 1 can also be stated in $L^1$ with slightly different assumptions \citep{acosta2004optimal}, it is also possible to derive a result similar to Theorem 3 purely in the $L^1$ norm. 
\end{rem}

The proposed method quantifies the model accuracy along a trajectory, which depends on the initial point $x(0)$. Since we are interested in obtaining results, which are representative in the whole domain of the training data $\mathcal{S}_c$, we propose to randomize $x(0)$ within $\mathcal{S}_c$ and average over the obtained results to make them meaningful for the whole domain.

\section{Experiments: Model Selection for Dynamical Systems}
\label{subsec:experiment_modelselection}
After discussing the capabilities of \licds as a model selection criterion, we also present empirical results in order to provide more evidence to our claims.
First, consider a dynamical system as in Eq.~\ref{eq:dynamical-system}, with $f(x) = \minus\tanh(x) + 0.1\sin(5x)$ and additive white noise $0.01\epsilon(t)$. This system is used to generate 10 noisy trajectories starting from randomly chosen initial points and 100 data points each are sampled with the aid of the Euler–Maruyama method \citep{schuss1988stochastic}. 
This results in a training set $S_c$ with total size of 1000 samples. In \fig \ref{fig:synthetic_model_selection}, the learned functions $\hat{f}^i_{\rm NN}$ are depicted together with the true function $f$. 
Figure~\ref{fig:synthetic_model_selection} clearly shows the connection between the $L$ score and the respective fit of the different NN architectures as the least score gives the best fit. 

Next, we consider dynamical systems of Eq.~\ref{eq:dynamical-system} for some benchmark problems.
%, which are motivated by physical systems. 
Similar benchmark problems are considered, for example, in \citep{doerr2017optimizing} and \citep{kroll2014benchmark}, which we used to shape the nonlinearities for the problems considered herein, which are summarized in Table \ref{tbl:model_selection}. After generating noisy data based on the dynamical system, we deploy several NNs with different depth and width and a GP to capture the behaviour of the system (details of the learning procedures in the supplementary material). 
%We refer to the supplementary material for more details on the specific learning procedure. 
The results in terms of the $L$ score and actual distance to the underlying function (in the $L^2$ norm sense) are shown in Table \ref{tbl:model_selection}. We emphasize that the \emph{best} learned function achieves the \emph{lowest} $L$ score.

%We do not claim that our algorithm improves any of these methods. 
The proposed method does not aim at improving any of the model learning methods.
Instead, we provide a structured way to post-process learned models and select the \emph{best} among several candidates. However, the presented ideas might be incorporated into improving also the training process.

\setlength\tabcolsep{1.5pt} % default value: 6pt

\begin{table}[b]
\tiny
\centering
\caption{\small Results for the model selection experiment. The proposed criterion \licds is applied to the dynamical functions $f$ shown in the first column. 
Columns 2--4 show results for learning NNs (NN=[\#number of hidden units of each layer]), and column 5 the GP.
%the following columns correspond to the particular architecture of the learning algorithm (NN=[\#number of hidden units of each layer]).
Per model, the two numbers are the \licds score (left) and the ground truth as the $L^2$ norm between the learned and the true function.
%We first state the corresponding \licds score and afterwards the ground truth in terms of the $L^2$ norm between the learned and the true function.
Best scores are highlighted in bold. 
%Last column corresponds to the dynamics function which is learned by GP. Each value is computed by averaging 10 runs of training\ar{Friedrich: could you check the caption to see if it is clear?}.
}
\label{my-label}
\resizebox{\columnwidth}{!}{%
\begin{tabular}{|c|c|c|c|c|c|c|c|c|}
\hline
\multirow{2}{*}{$\minus\tanh(x)$} & \multicolumn{2}{c|}{NN=[1]} & \multicolumn{2}{c|}{NN=[10]} & \multicolumn{2}{c|}{NN=[40]} & \multicolumn{2}{c|}{GP}\\ \cline{2-9} 
                   & $\mathbf{4.40\mathrm{e}-5}$            &     $\mathbf{5.84\mathrm{e}-5}$        &  $7.43\mathrm{e}-5$         &   $6.15\mathrm{e}-4$       &    $8.69\mathrm{e}-5$       &      $2.1\mathrm{e}-3$  &    $7.17\mathrm{e}-5$       &      $2.52\mathrm{e}-4$\\ \hline
\multirow{2}{*}{$\minus\tanh(x) + 0.5x$}  & \multicolumn{2}{c|}{NN=[1]}          & \multicolumn{2}{c|}{NN=[2]} & \multicolumn{2}{c|}{NN=[5]}& \multicolumn{2}{c|}{GP}\\ \cline{2-9} 
                   & $3.52\mathrm{e}-3$             &     $3.08\mathrm{e}0$          &    $2.53\mathrm{e}-3$        &    $4.46\mathrm{e}-1$       &    $\mathbf{6.20\mathrm{e}-4}$        &    $\mathbf{4.01\mathrm{e}-2}$ &    $9.01\mathrm{e}-4$       &      $2.16\mathrm{e}-1$\\ \hline
\multirow{2}{*}{$\minus x/(1+x^2)$}  & \multicolumn{2}{c|}{NN=[1]}          & \multicolumn{2}{c|}{NN=[2]} & \multicolumn{2}{c|}{NN=[30]}& \multicolumn{2}{c|}{GP}\\ \cline{2-9} 
                   &   $1.59\mathrm{e}-4$              &     $0.49\mathrm{e}-1$          &    $\mathbf{5.077\mathrm{e}-5}$        &    $\mathbf{1.26\mathrm{-4}}$       &      $7.74\mathrm{e}-5$      &      $3.03\mathrm{e}-4$     &    $1.66\mathrm{e}-4$       &      $3.63\mathrm{e}-3$\\ \hline
\multirow{2}{*}{$\minus\tanh(x)+0.5\sin(x)$}  & \multicolumn{2}{c|}{NN=[1]}          & \multicolumn{2}{c|}{NN=[2]} & \multicolumn{2}{c|}{NN=[5]}& \multicolumn{2}{c|}{GP}\\ \cline{2-9} 
                   &     $2.03\mathrm{e}-4$           &         $1.78\mathrm{e}0$       &       $1.99\mathrm{e}-4$     &     $1.13\mathrm{e}0$      &  $\mathbf{2.51\mathrm{e}-5}$          &   $\mathbf{3.01\mathrm{e}-1}$      &    $9.16\mathrm{e}-4$       &      $2.48\mathrm{e}0$\\ \hline
\end{tabular}
}
\label{tbl:model_selection}
\end{table}

\section{Discussion}
In this paper, we proposed \licds as a method to efficiently encode information of a dynamical system, which is either known or learned from a sequence of observations. We built the encoding scheme on top of the minimum message length principle and came up with a practical method to approximate the algorithmic complexity of dynamical systems by means of local approximations.  In addition to efficient encoding, we showed through experiments and theorems that the proposed encoding criterion can be used for model selection likewise.  By comparing \licds scores for different learned models (e.g., NN and GP), the model that is closer to the underlying dynamics can be selected.  For future work, we aim to apply \licds for efficient communication in networked multi-agent systems.  Also, we seek to characterize more precisely the class of dynamical systems, for which \licds is effective (\cf Remark \ref{remark:aboutTheorem}), and investigate extensions to stochastic systems. 
% \seb{I'm not sure what the following sencence means (I would omit, but feel free to reword and add if you prefer):
% For the future work, we would like to connect this idea to inference and check whether the best compressed model is the one that better captures the physical couplings of the system.}\ar{It was related to causality. Let's remove it for now.}

%Efficient encoding has numerous applications in multi agent systems when dynamics information must be exchanged among the agents. We also showed that the proposed method can be used as a criterion for model selection. We showed that computing \licds criterion for learned models can be used to choose one learned model over another. 

\section*{Acknowledgments}
This work was supported in part by the Max Planck Society, the Cyber Valley Initiative, and the German Research Foundation (DFG) grant TR 1433/1-1.

\bibliographystyle{unsrt}      
\bibliography{arxiv}
% \newpage

\newpage
\section*{Appendices}
\appendix 

\section{Algorithmic Complexity and Universal Turing Machine}
The concepts of algorithmic complexity and universal Turing machine were briefly introduced in paper and will be presented in more details here. Using the Alice-Bob scenario, we briefly review some necessary terms. The \emph{message} is a sequence of symbols chosen from a set of alphabets. Each symbol is encoded by a binary sub-sequence called \emph{word} that forms the entire message when the code for all symbols are pieced together. Shannon's information theory considers the message as a sequence of outcomes of a random process~\citep{shannon1948mathematical}. Assume the message can obtain its words from a set $\{a_i : i= 1,\ldots,N\}$ and the probability of word $a_i$ be $p_i$ for all $i$. The goal is to find a code that maps each word $a_i$ to a binary string with length $l_i$ such that the expected length of the string which is defined as 
\begin{equation}
\label{eq:expected_code_length}
E(l) = \sum_i p_i l_i
\end{equation}
is minimized. It can be proved that the optimal code in this sense will be obtained if $l_i=-\log p_i$ where this value is also known as Shannon's entropy. We consider the base of the logarithm 2 throughout this paper. The length $l_i$ of the code of a word $a_i$ can be taken as a measure of information content of word $a_i$ represented by ${\rm Info}(a_i)$. Nonetheless, the major limitation of Shannon's approach to information is its explicit dependence on probabilistic source of the message. Algorithmic Complexity(AC) is a different approach that removes this assumption and gives a more generic idea of information. To present the core idea of AC, some preliminary definitions are required which will be briefed in the following

{\it Universal Turing Machine--- }A Turing machine (TM) is a machine with 
\begin{enumerate}
\item A \emph{clock} that synchronizes all activities of the machine.
\item A finite set of internal states indexed by $\{1,2,\ldots,S\}$. The machine may change its state at the clock tick.
\item A binary \emph{work tape} which can be moved to the right or left and be updated by the machine.
\item A one-way binary \emph{input tape} which forms the input to the machine. The input tape cannot be moved backward.
\item A one-way binary \emph{output tape} that carries the machine's output.
\item An instruction list that determines the action of the machine at each clock tick depending on the current value of the input tape, work tape and the internal state of the machine. The action may include moving the input tape, updating and moving the work tape, updating and moving the output tape, or moving to a new internal state.
\end{enumerate}

Given a binary string $A$ representing some data or information, the amount of information, a.k.a Algorithmic Complexity (AC), in $A$ given a particular Turing Machine (TM), is the shortest input tape $I$ which will cause TM to output $A$, i.e., ${\rm AC}(A)=I$. It is obvious from this definition that the information content of a message $A$ depends on the chosen TM. The concept of Universal Turing Machine (UTM) comes as an assistance here. Apart from its detailed definition that can be looked up in~\citep{kolmogorov1965three}, a UTM has the interesting property of being programmable. Meaning that the input tape $I$ may consist of two concatenated parts $I=I_1.I_2$ such that $I_1$ pushes the the initial Turing machine TM\textsubscript{0} into a state from that state on, the UTM behaves as another Turing machine TM\textsubscript{1}. The second part of the input tape $I_2$ is then decoded by TM\textsubscript{1} rather than TM\textsubscript{0}. This capacity of UTM enables us to achieve a universal measure for complexity or information content. In the next section we discuss how information content of dynamical systems can be described in the framework of a UTM.

\section{Proofs of Theoretical Results}
We provide here the proofs to our theoretical results:
\begin{theo}
\label{thm:closeFunctionsToCloseX}
Consider Eq.~(1) with $f$ Lipschitz-continuous on $[0,T]$.
%Assume we consider a system like in Eq.~\ref{eq:dynamical-system},
%where $f$ is a Lipschitz-continuous function on $[0,T]$. 
Furthermore, assume a Lipschitz-continuous approximation $\hat{f}$ is used to obtain state approximations $\hat{x}$. Then, for $\hat{x}(0) = x(0) = x_0$,
%we can bound the error between $x$ and $\hat{x}$ and show
\begin{equation}
    \| x(.) - \hat{x}(.) \|_{L^2([0,T])}^2 \leq T^2 \| f(x(.)) - \hat{f}(x(.)) \|_{L^2([0,T])}^2.
\end{equation}
In particular, this implies: if $\| f(x(.)) - \hat{f}(x(.)) \|_{L^2([0,T])}^2 \leq \epsilon$, then  $\| x(.) - \hat{x}(.) \|_{L^2([0,T])}^2 \leq T^2 \epsilon$.
\end{theo}

\begin{proof}[Proof of Theorem 1]
We start the proof by showing that there exists a well defined solution $x(t)$ to the considered ODE, which is due to the Picard–Lindel\"{o}f theorem. 

Next, we show how to bound a function against its derivative, which is frequently done in Poincaré inequalities. Depending on the given assumptions, these results all look slightly different.
Here, we use $z \in C^1((0,T))$, $z(0) = 0$ and proof $\| z(.)\|_{L^2([0,T])}^2 \leq T^2 \| z'(.)\|_{L^2([0,T])}^2$.

We start with the fundamental theorem of calculus and obtain
\begin{equation}
z(t) = \int_0^t z'(s) \mathrm{d}s, ~ \forall t \in [0,T].
\end{equation}
Hence, we obtain for the absolute value
\begin{equation}
|z(t)| \leq \int_0^t |z'(s)| \mathrm{d}s.
\end{equation}
Now we assume a multiplicative one and apply the Cauchy-Schwarz inequality
\begin{equation}
\int_0^t |z'(s)| 1 \mathrm{d}s \leq \sqrt{\int_0^t |z'(s)|^2 \mathrm{d}s} \sqrt{\int_0^t 1 \mathrm{d}s}.
\end{equation}
Since $t\leq T$ and everything is nonnegative, we obtain
\begin{equation}
|z(t)| \leq \sqrt{T} \sqrt{\int_0^T |z'(s)|^2 \mathrm{d}s} .
\end{equation}
Taking the square and integrating does not change the inequality, since the right hand side is not dependent on $t$ anymore. This yields the final result
\begin{equation}
\int_0^T |z(s)|^2 \mathrm{d}s \leq T^2 \int_0^T |z'(s)|^2 \mathrm{d}s.
\end{equation}
Now we substitute $z(t) = x(t) - \hat{x}(t)$ and obtain
\begin{equation}
 \| x(.) - \hat{x}(.) \|_{L^2([0,T])}^2 \leq T^2 \| f(x(.)) - \hat{f}(x(.)) \|_{L^2([0,T])}^2 \leq T^2 \epsilon.
\end{equation}
\end{proof}

\begin{lem}
Assume $x(t) \in C^1((0,T))$ and $\sup\limits_{t \in (0,T)} |x'(t)|\leq \tilde{M}$ on the domain $[0,T]$. 
We can show
\begin{equation}
\|x(.)\|_{L^\infty([0,T])} \leq K \|x(.)\|_{L^1([0,T])}.
\end{equation}
\end{lem}
\begin{proof}
Since we consider a bounded domain and the derivative is bounded we conclude that $\sup\limits_{t \in (0,T)}|x(t)| = M$ is bounded as well.
We proof the statement by considering the worst case scenario, which is a triangle for this case.
What essentially can happen is that the support of the function shrinks, while the maximum remains constant. However, by bounding the derivative we have 
control over the growth of the area beneath the function. Therefore, the extreme case is a triangle $\phi_\Delta$ with the maximal slope $\tilde{M}$ and peak point $M$. This yields
\begin{equation}
\|\phi_\Delta \|_{L^\infty} = M
\end{equation}
and 
\begin{equation}
\|\phi_\Delta \|_{L^1} = \frac{1}{2}\left(\frac{M^2}{\tilde{M}^2}+M \right).
\end{equation}
Hence for $\frac{1}{K}\leq \frac{1}{2}\left(\frac{M}{\tilde{M}^2}+1 \right)$ we obtain our claim.
\end{proof}

\begin{lem}
Assume the function $x(t)$ is monotonically increasing in $[0,T]$. Then the variation of $x(t)$ is given by 
\begin{equation}
V(x, [0,T]) = x(T) - x(0).
\end{equation}
\end{lem}
\begin{proof}
The proof is straight forward and follows immediately with a monotonicity and telescope sum argument.
\end{proof}
% \begin{rem}
% In general, norms in $L^p(\mathbb{R}^n)$ function spaces are not equivalent and the above result holds due to the restrictive assumptions on the functions.
% \end{rem}

\begin{theo}
Let the assumptions of Theorem~\ref{thm:closeFunctionsToCloseX} hold.  
\newline
Additionally, assume 
%We consider the same setting as in Theorem 1, but additionally assume that
$\sup\nolimits_{t \in (0,T)} x(t) = M < \infty$ and $\sup\nolimits_{t \in (0,T)} x'(t) \leq \tilde{M}$. 
%In particular, states are not allowed to oscillate arbitrary fast. 
Then, we have 
\begin{equation}
\| f(x(.)) - \hat{f}(x(.)) \|_{L^1([0,T])} \leq C \| x(.) - \hat{x}(.) \|_{L^1([0,T])} .
\end{equation}
This implies: if  $\| x(.) - \hat{x}(.) \|_{L^1([0,T])} \leq \delta$ then $\| f(x(.)) - \hat{f}(x(.)) \|_{L^1([0,T])} \leq C\delta$. 
\end{theo}

\begin{proof}[Proof of Theorem 2]
We use bounded variation type arguments here. In particular, we start again with
\begin{equation}
\int\limits_0^T |f(x(t))|\mathrm{d}t = \int\limits_0^T |x'(t)|\mathrm{d}t.
\end{equation}
It is well known fact in analysis that the quantity $\int\limits_0^T |x'(t)|\mathrm{d}t$ can be used to compute the total variation of a smooth function. We will use an equivalent approach to quantify the total variation and use this to bound the derivative with the states. We use
\begin{equation}
\int\limits_0^T |x'(t)|\mathrm{d}t = \sup\limits_{P \in \mathcal{P}} \sum\limits_{i=0}^{n_P - 1} |x(t_{i+1}) - x(t_i)|,
\end{equation}
where we take the supremum over all possible grids, which are not necessarily equidistant. Hence, $n_P$ is the number of grid points, which can in general go to infinity. This is even possible for functions with bounded variation, as long as the function value decays fast enough. 
%Well investigated examples with unbounded variation are $g(x) = \sin(\frac{1}{x})$ and $h(x) = x\sin(\frac{1}{x})$. The function $x^2\sin(\frac{1}{x})$ also oscillates infinity fast towards zero, however it has bounded variation. 
The assumption $\sup\limits_{t \in (0,T)} |x'(t)| < \infty$ ensures a finite number of oscillations on a bounded domain, which combined with Lemma 2 yields that $n_P < \infty$. Hence, there exists an optimal grid with a finite number of points $n_P$ and we can use the bound 
\begin{equation}
\sup\limits_{P \in \mathcal{P}} \sum\limits_{i=0}^{n_P - 1} |x(t_{i+1}) - x(t_i)| \leq n_P \sup\limits_{i \in P^*} |x(t_{i+1}) - x(t_i)|.
\end{equation}
We can split this apart with the triangle inequality and make the quantity even bigger by dropping the grid. Hence, we obtain
\begin{equation}
\| f(x(.)) \|_{L^1([0,T])} \leq 2n_P \|x(.)\|_{L^\infty([0,T])}.
\end{equation}
With the aid of Lemma 1, $C = 2 n_P K$ and the same argument as in the end of the proof of Theorem 1 we conlude this proof.
\end{proof}

\begin{theo}[Model Selection]
If the previous assumptions hold, $L(\hat{f}_1) \leq L(\hat{f}_2)$ and 
\begin{equation}
    \| \hat{f}_1(x(.)) - \tilde{f}_1(x(.)) \|_{L^1([0,T])} \leq C \lambda k_1^*, 
    \label{eq:conditionMainThrm}
\end{equation}
then 
\begin{equation}
\| f(x(.)) - \hat{f}_1(\hat{x}_1(.)) \|_{L^1([0,T])} \leq E + \xi \| f(x(.)) - \hat{f}_2(\hat{x}_2(.)) \|_{L^2([0,T])},    
\end{equation}
where $E, \xi, C \geq 0$ are constants, which depend on certain properties of the dynamical systems.
\end{theo}
\begin{proof}[Proof of Theorem 3]
We consider three objects in this proof - the true dynamical function $f$, an approximation $\hat{f}$, which is most likely obtained from a learning algorithm and the local approximation $\tilde{f}$, obtained through a local expansion, e.g. Taylor.

We start by inserting $+\tilde{f}- \tilde{f}$ and obtain with the triangle inequality
\begin{equation}
\| f(x(.)) - \hat{f}_1(x(.)) \|_{L^1([0,T])} \leq \nu_1 +  \| f(x(.)) - \tilde{f}_1(x(.)) \|_{L^1([0,T])}.
\end{equation}
Now, we use Theorem 2 and obtain
\begin{equation}
 \nu_1 +  \| f(x(.)) - \tilde{f}_1(x(.)) \|_{L^1([0,T])} \leq \nu_1 + C \|x(.) - \tilde{x}_1(.)\|_{L^1}.
\end{equation}
The assumption $L(\hat{f}_1) \leq L(\hat{f}_2)$ expands to
\begin{equation}
\lambda k_1^* + \| x(.) - \tilde{x}_1(.) \|_{L^1([0,T])} \leq \lambda k_2^* + \| x(.) - \tilde{x}_2(.) \|_{L^1([0,T])}.
\end{equation}
Hence, for $\nu_1 \leq C\lambda k_1^*$ we obtain 
\begin{equation}
 \nu_1 + C \|x(.) - \tilde{x}_1(.)\|_{L^1} \leq C \lambda k_2^* + C\| x(.) - \tilde{x}_2(.) \|_{L^1([0,T])}.
\end{equation}
Now we apply the Cauchy-Schwarz inequality to transform the $L^1$ norm into the $L^2$ norm and obtain
\begin{equation}
C\lambda k_2^* + C\| x(.) - \tilde{x}_2(.) \|_{L^1([0,T])} \leq C\lambda k_2^* + T C\| x(.) - \tilde{x}_2(.) \|_{L^2([0,T])} .
\end{equation}
With the aid of Theorem 1 it follows that
\begin{equation}
C\lambda k_2^* + T C\| x(.) - \tilde{x}_2(.) \|_{L^2([0,T])} \leq C\lambda k_2^* + T^2 C\| f(x(.)) - \tilde{f}_2(x(.)) \|_{L^2([0,T])}.
\end{equation}
With the aid of the triangle inequality we can again show
\begin{equation}
C\lambda k_2^* + T^2 C\| f(x(.)) - \tilde{f}_2(x(.)) \|_{L^2([0,T])} \leq C\lambda k_2^* + T^2 C \nu_2 + T^2 C \| f(x(.)) - \hat{f}_2(x(.)) \|_{L^2([0,T])} 
\end{equation}
\end{proof}

\section{Learning dynamical systems}
The more detailed description of the method we used to learn dynamics function $\hat{f}$ from observational data is presented here. We use a simple black-box approach to learn the dynamical system from a set of trajectories.

{\it Learning $\hat{f}$ by Neural network---} Assume we are given a collection of sequences of observations $S_c=\{S_1, S_2, \ldots\}$. Each sequence $S_i$ covers a trajectory in the state space starting from some starting point $x(0)$. We use each sequence $S_i$ as a mini-batch of observations and train $\hat{f}(x;\theta)$ by the following simple relationship between its input/output pairs:
\begin{align}
\dot{x}(t)=f(x(t)) &\implies \dot{x}(t)=f(x(t);\theta) \\
&\implies \frac{\Delta x}{\Delta t}=f(x(t);\theta)\\
&\implies x((k+1)T_s)-x(kT_s)=\hat{f}(x(kT_s);\theta)
\label{eq:learning_dynamics}
\end{align}

The reason for using a collection $\mathcal{S}_c$ instead of a single sequence $S$ is clear. A single sequence starting from an initial point $x(0)$ is unlikely to be representative enough so that $\hat{f}$ is learned as a good approximation to $f$. Once $\hat{f}$ is learned, we can use automatic differentiation to compute its derivative w.r.t. the input~\citep{abadi2016tensorflow}.

{\it Learning $\hat{f}$ Gaussian Process---} 
The discretization of the nonlinear dynamics function is done just like above. We used a vanilla GP without any sparse approximations and a squared exponential kernel function.

\section{More experiments}
Some parts of the experiment sections are delegated to here from the main text. It includes more sophisticated experiments with higher dimensional and physical dynamical systems.

\subsection{Enlarged version of the illustrative example}
\begin{figure}[t!]
	\centering
	\hspace{-0.3cm}
\subfigure[State trajectory starting from $x(0)=2$]{
\includegraphics[width=0.45\linewidth]{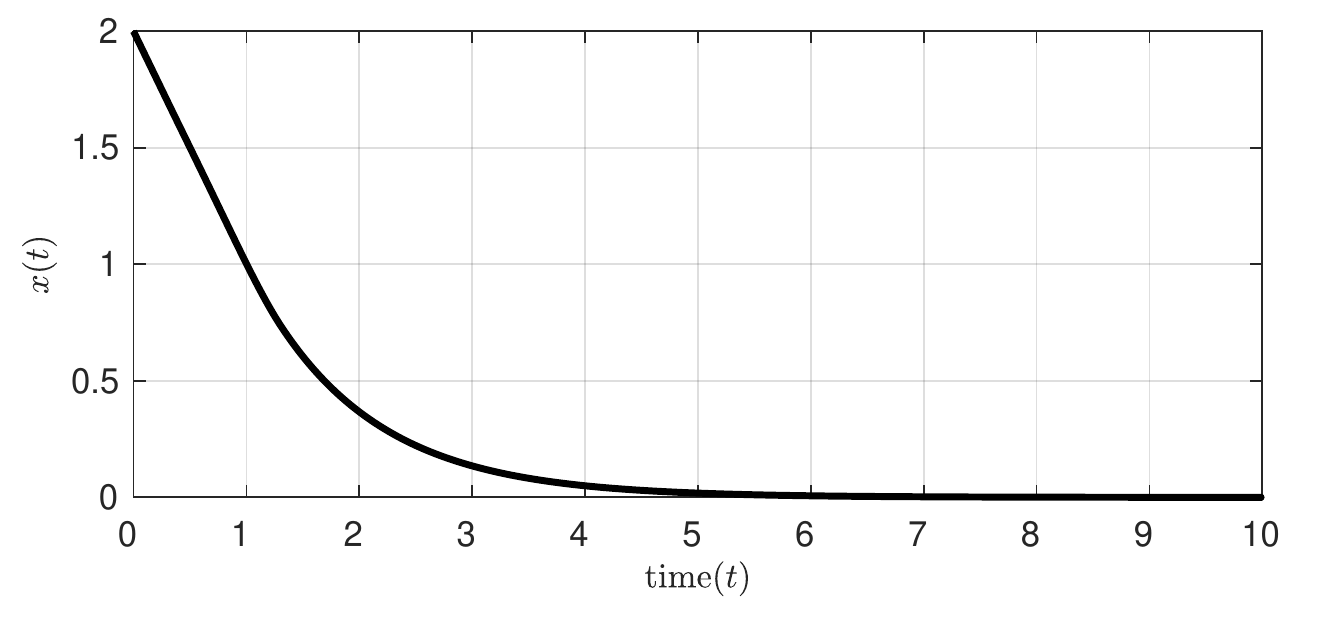}}
\subfigure[True dynamics $f$ vs learned dynamics $\hat{f}$]{
\includegraphics[width=0.45\linewidth]{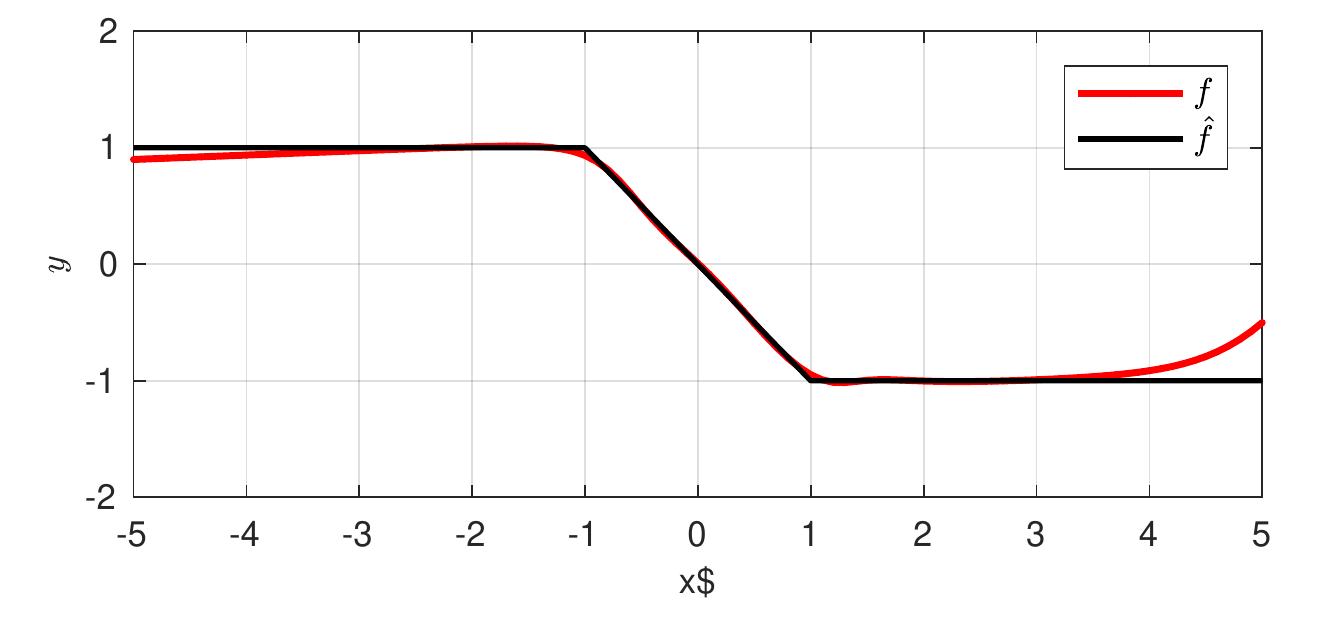}}\\
\caption{(a) Shows that the negative region of the state space is not visited when the initial state is positive. (b) Shows the learned dynamics function $\hat{f}$ when only some trajectories of the state dynamics governed by $f$ is available.}
\label{fig:sat_state_function}
\end{figure}

 As an illustrative example, let's assume the dynamical system $\dot{x}(t)=f(x(t))$ with $f(x)=\minus{\rm sat}(x)$ depicted in Fig.~\ref{fig:sat_state_function}(a).
This dynamical system is stable and the evolution of its state is depicted in Fig.~\ref{fig:sat_state_function}. If the initial point $x(0)$ resides in the positive region of the state space, it never leaves the non-negative side of the state space. Hence, encoding $f$ for the negative domain is not necessary and we can safely only encode $f$ in its positive domain. This can be formalized in terms of algorithmic complexity as
\begin{equation}
{\rm AC}(x(0), f) < {\rm AC}(x(0)) + {\rm AC}(f)
\label{eq:AC_state_model}
\end{equation}
meaning that knowing $x(0)$ allows to design a better code for $f$.

\subsection{Illustrative example:}
Here is more figures related to sec.3(Fig.2) of the main text. The assumed dynamical system is
$\dot{x}(t)=-\tanh(x(t))$. The state evolution for different number of partitions is depicted in Fig.~\ref{fig:exp_tanh_encoding}. Fig.~\ref{fig:exp_tanh_encoding}(f) shows that \licds score finds a non-trivial local encoding of the space for $m>1$.

\subsection{Local apprximations to the learned function:}
In this section, the system $\dot{x}(t)=\tanh(x(t)) + 0.1\sin(5x(t))$ is used to generate samples based on the method explained for the experiment in the main text. The dynamics function is then learned and depicted in Fig.~\ref{fig:multiple_taylors}(a). Once the function is learned, multiple local Taylor approximations is computed and shown in Fig.~\ref{fig:multiple_taylors}(b-e) corresponding to different number of partitions. Figure.~\ref{fig:multiple_taylors}(f) shows the $L$ score is optimal for a non-trivial $m>1$ case.

\begin{figure}[tb]
	\centering
	\hspace{-0.3cm}
\subfigure[$m=1$]{
\includegraphics[width=0.3\linewidth]{./figures_arxiv/nn_part_1_state_1.pdf}}
\subfigure[$m=2$]{
\includegraphics[width=0.3\linewidth]{./figures_arxiv/nn_part_2_state_1.pdf}}
\subfigure[$m=3$]{
\includegraphics[width=0.3\linewidth]{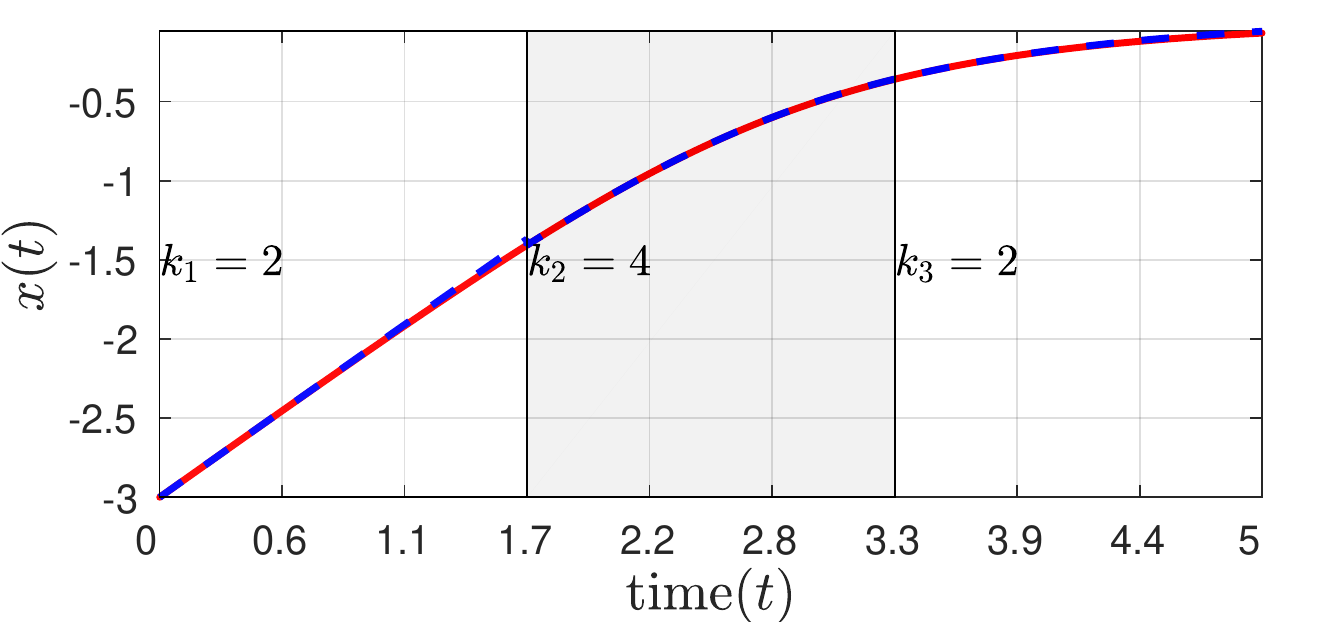}}\\
\subfigure[$m=4$]{
\includegraphics[width=0.3\linewidth]{./figures_arxiv/nn_part_4_state_1.pdf}}
\subfigure[$m=5$]{
\includegraphics[width=0.3\linewidth]{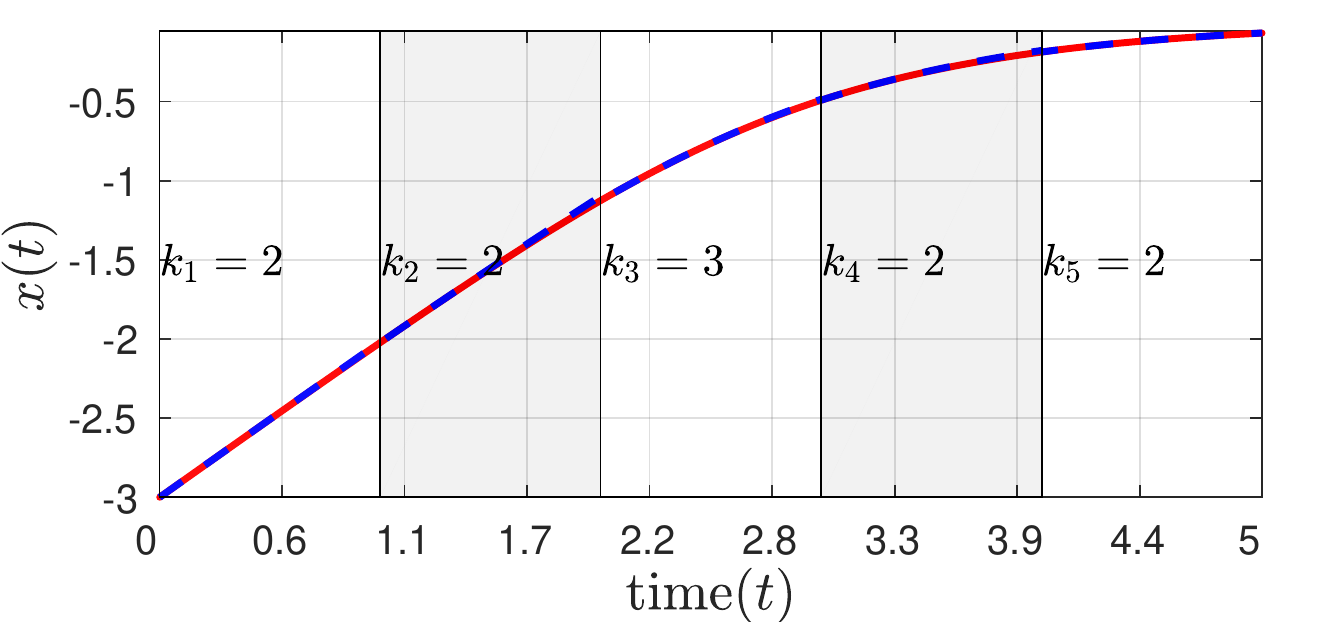}}
\subfigure[\licds score]{
\includegraphics[width=0.3\linewidth]{./figures_arxiv/Cost.pdf}}
	\caption{The effect of learning $f$ by $\hat{f}$ then approximating $\hat{f}$ by $\tilde{f}$ on state evolution of the dynamical system. The optimal trade-off is found by \licds score in (f) for $m=2$ and the corresponding state evolution is shown in (b). The number of employed basis functions $k$ in each local region is written over that region} 
	\vspace{-0.2cm}
\label{fig:exp_tanh_encoding}
\end{figure}

\begin{figure}[t!]
	\centering
% 	\hspace{-0.3cm}
\subfigure[True and learned dynamics]{
\includegraphics[width=0.3\linewidth]{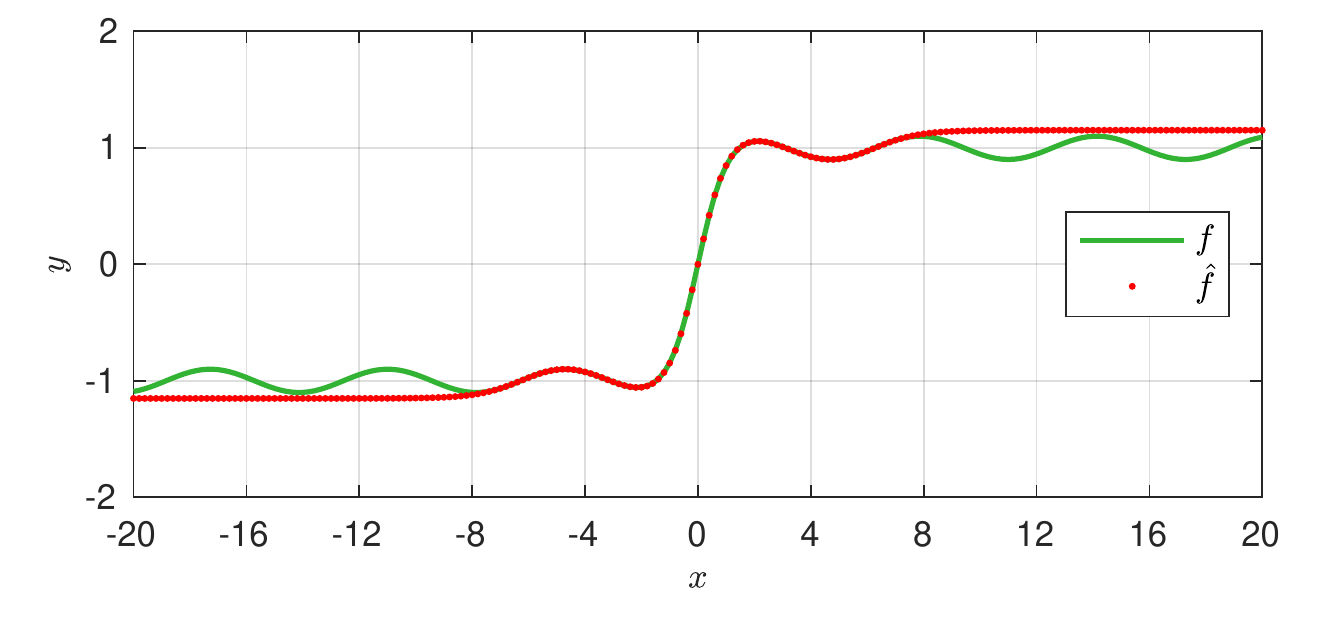}}
\subfigure[$m=1$, $L = 0.51$]{
\includegraphics[width=0.3\linewidth]{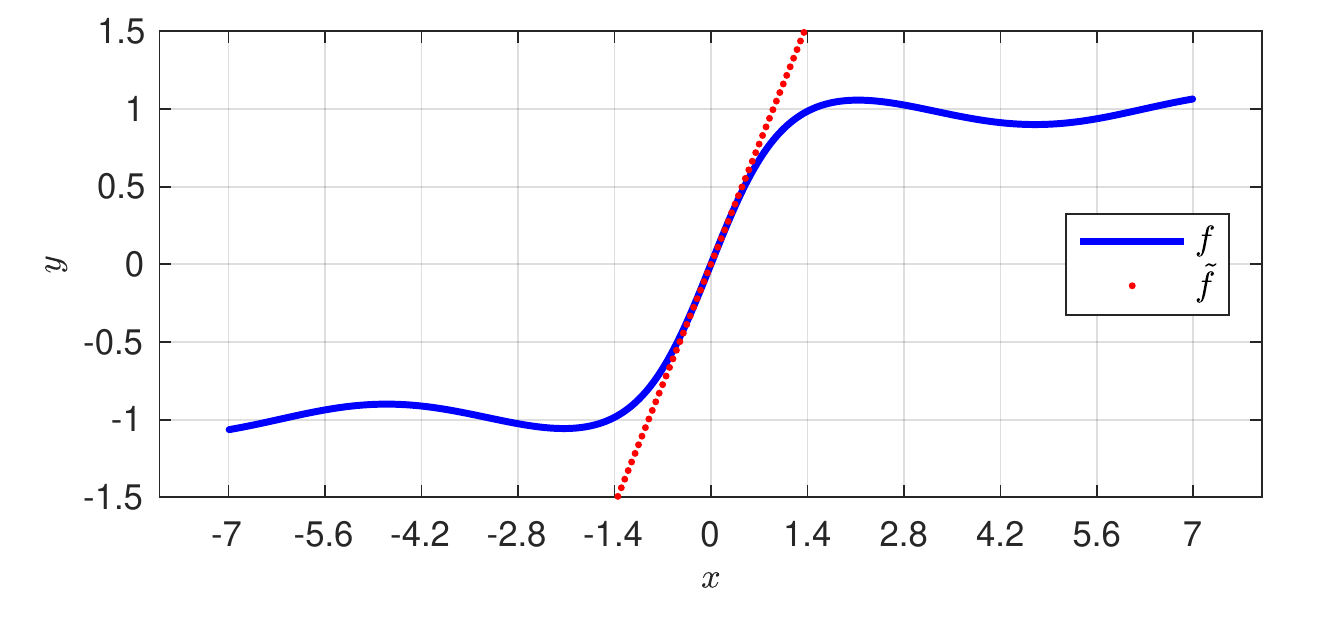}}
\subfigure[$m=2$, $L = 0.42$]{
\includegraphics[width=0.3\linewidth]{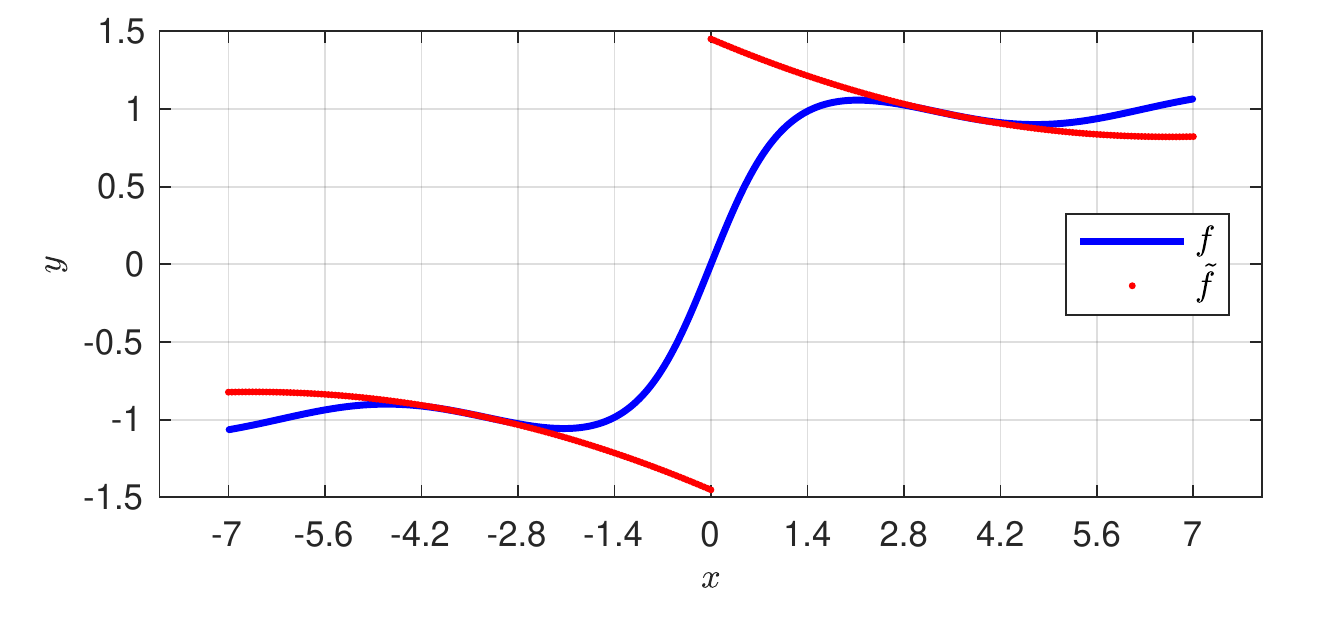}}\\
\subfigure[$m=3$, sections, $L = 0.25$]{
\includegraphics[width=0.3\linewidth]{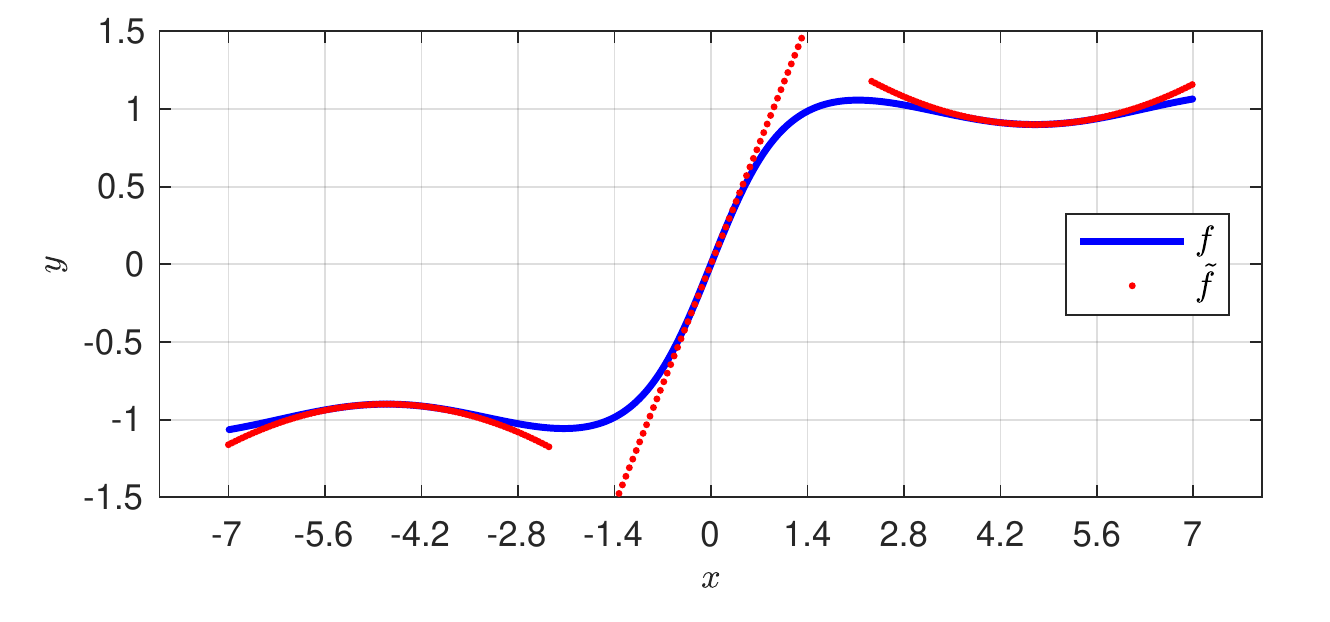}}
\subfigure[$m=4$, sections, $L = 0.36$]{
\includegraphics[width=0.3\linewidth]{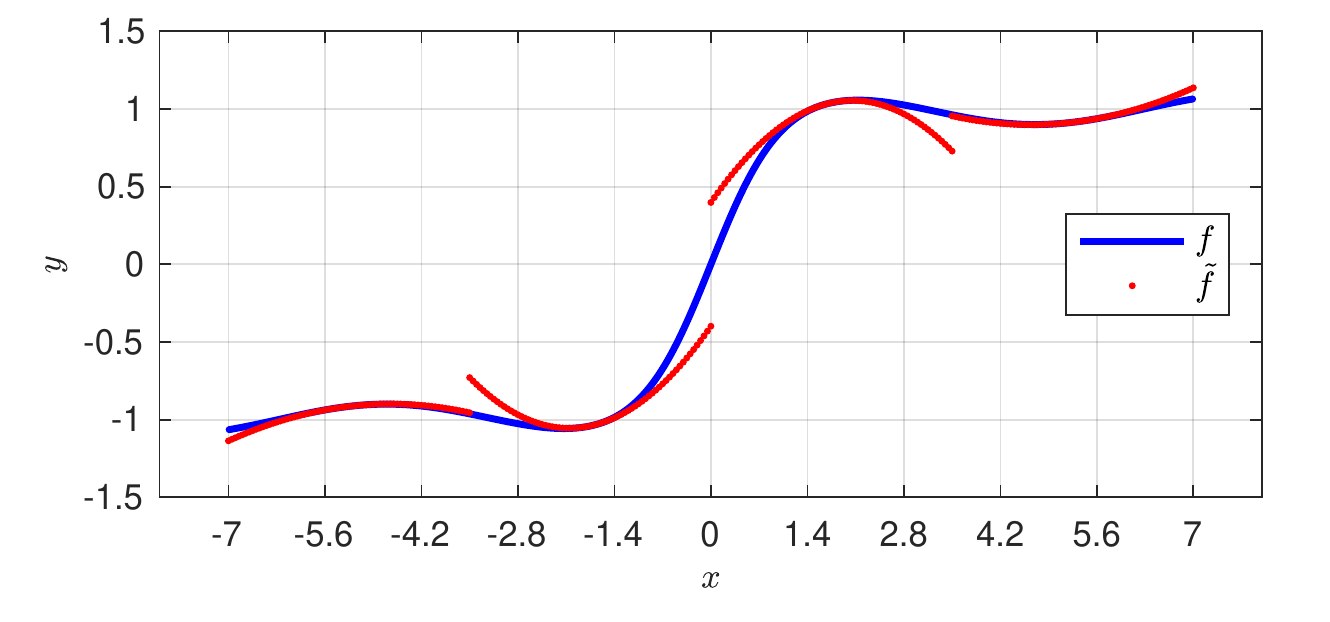}}
\subfigure[$m=5$, sections, $L = 0.55$]{
\includegraphics[width=0.3\linewidth]{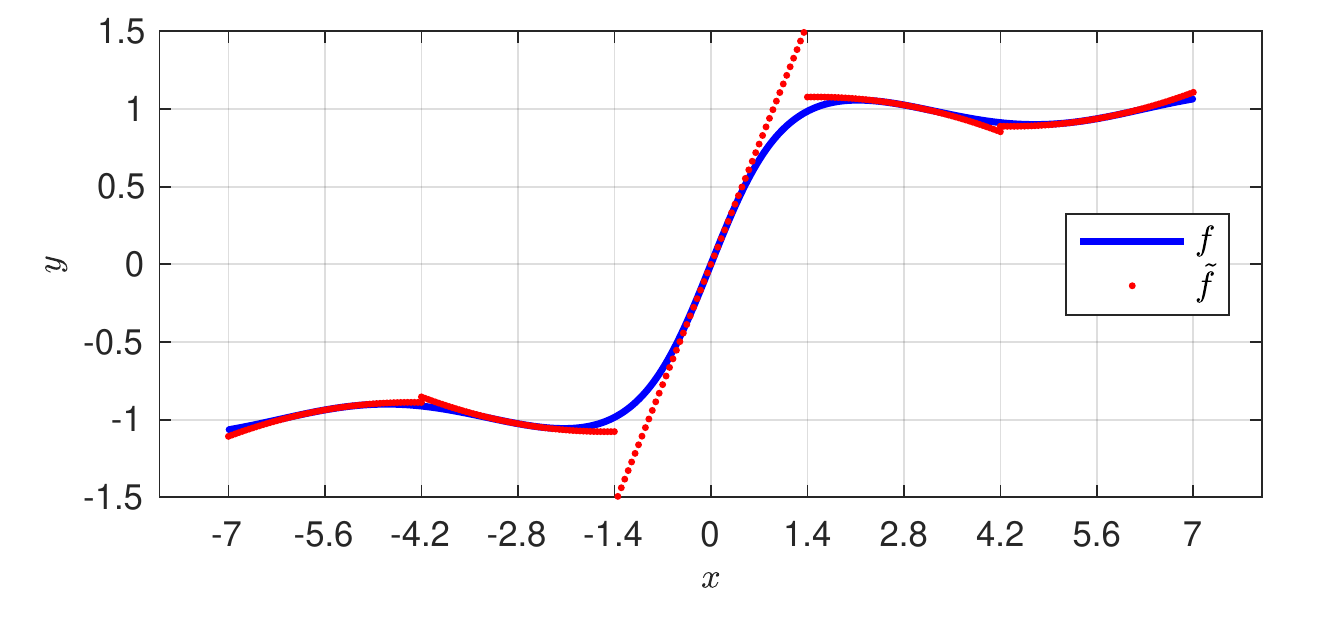}}
	\caption{learned function with various number of local approximations}
	\vspace{-0.2cm}
\label{fig:multiple_taylors}
\end{figure}

\subsection{Pendulum dynamics:}
Part of this experiment is in the main text. This is the complete version. We consider a realistic physical system of a pendulum with two dimensional dynamics $\dot{x}_1 = x_2$ and $\dot{x}_2 = -x_2 - 9.81{\rm sin}(x_1)$. The result for how $\hat{f}$ is learned from $f$ and how $\tilde{f}$ approximates $\hat{f}$ is depicted in Fig.\ref{fig:pendulum_true_vs_learned_2d}.
The corresponding result in the state space is shown in Fig.~\ref{fig:exp_pendulum}, where $x_1$ is the angular position and $x_2$ the angular velocity of the pendulum. Again, \licds finds a good trade-off between model complexity and prediction accuracy.
% We also did experiments for high dimensional realistic example of a Quad-rotor but we delegated it to the supplementary material due to the shortage of space.
% We consider he dynamics of a pendulum with two state equations $\dot{x}_1 = x_2$ and $\dot{x}_2 = -x_2 - 9.81{\rm sin}(x_1)$, where $x_1$ is the angular position and $x_2$ the angular velocity.
% \begin{align}
% &\dot{x}_1 = x_2\\
% &\dot{x}_2 = -rx_2 - \frac{g}{l}{\rm sin}x_1
% \label{eq:pendulum}
% \end{align}
% The system runs from the initial state $[x_1(0),x_2(0)]=[\frac{\pi}{4}, 0]$. 
% Assume polynomials with degrees up to $3$ are allowed as the set of basis functions to approximate each local region. Given this set, the approximation is done by Taylor expansion. The learned dynamics functions and their local approximations are depicted in Fig.~\ref{fig:pendulum_true_vs_learned_2d}. The optimal degree $k_i$ of the Taylor expansion corresponding to each local region can be seen in \fig \ref{fig:exp_pendulum}.
% This determines the number of terms in the expansion and also the number of coefficients that should be stored with a certain precision for that particular region. Hence, this degree gives a measure of local model complexity as described in sect.2 in of the main text in detail.

\subsection{Quadrotor dyamics:}
\licds as a method for compression, can be beneficial in any setting where it is necessary to transmit state information to a distant node. One typical example of these settings is Unmanned Aerial Vehicle (UAV), whose states are constantly measured on-board, but only occasionally transmitted to the ground base. Continuous transmission of data is expensive in terms of battery power and bandwidth. 
%This gave rise to event-triggered communication schemes, where states are only transmitted when required. (ST: omit since this example doesn't use event-triggering; avoiding wrong impression/expectations)
To show the performance of \licds on more complex dynamics, we test it on the deterministic dynamical system of a quad-rotor UAV~\citep{hoffmann2007quadrotor}.
%in this experiment, we choose to test \mieds on the deterministic dynamical system of a quad-rotor UAV \cite{hoffmann2007quadrotor}:
% \seb{What does $r[c(\phi)t(\theta)]$ mean?  Not sure how to read this notation...}\ar{I defined t(.). [] is just brackets}

In the following state space equation, $\phi\in[-\pi, \pi]$ is roll (rotation around x axis) and $\theta\in(-\frac{\pi}{2},\frac{\pi}{2})$ is pitch (rotation around y axis) in the earth frame. The value of angles can violate the bounds but its interpretation is circular. The vector $[u, v, w, p, q, r]$ contains the linear and angular velocities in the body frame. The functions $s$, $c$, and $t$ are sine, cosine, and tangent, respectively. The vector $[f_{wx}, f_{wy}, f_{wz}, f_t]=[1, 1, 1, 0]$ contains the wind forces and time-varying disturbance, $[\tau_{wx}, \tau_{wy}, \tau_{wz}]=[1, 1, 1]$ wind torques, and $[\tau_x, \tau_y, \tau_z]=[1, 1, 1]$ the control torques generated by the differences in the rotor speeds. The mass $m$ and gravity constant $g$ are set to $1$, $9.81$ respectively. We assume the system
\begin{equation}
\left\{
	\begin{array}{ll}
		\dot{\phi}=p+r[c(\phi)t(\theta)]+q[s(\phi)t(\theta)] \\
        \dot{\theta}=q[c(\theta)]-r[s(\phi)]\\
%         \dot{\psi}=r\frac{c(\phi)}{c(\theta)}+q\frac{s(\phi)}{c(\theta)}\\
        \dot{p}=\frac{I_y-I_z}{I_x}rq+\frac{\tau_x+\tau_{wx}}
        {I_x}\\
        \dot{q}=\frac{I_z-I_x}{I_y}pr+\frac{\tau_y+\tau_{wy}}
        {I_y}\\
        \dot{r}=\frac{I_x-I_y}{I_z}pq+\frac{\tau_z+\tau_{wz}}{I_z}\\
        \dot{u}=rv-qw-g[s(\theta)]+\frac{f_{wx}}{m}\\
        \dot{v}=pw-ru+g[s(\phi)c(\theta)]+\frac{f_{wy}}{m}\\
        \dot{w}=qu-pv+g[c(\theta)c(\phi)]+\frac{f_{wz}-f_t}{m}\\   
	\end{array}
\right.
\end{equation}
runs from the initial state $[{-2},{-3},1,3,1,4,2,1]$.
%and the initial state is set to $[1,-2,3,1,-3,4,2,1]$

\begin{figure}[t!]
	\centering
	\hspace{-0.3cm}
\subfigure[True and learned dynamics]{
\includegraphics[width=0.45\linewidth]{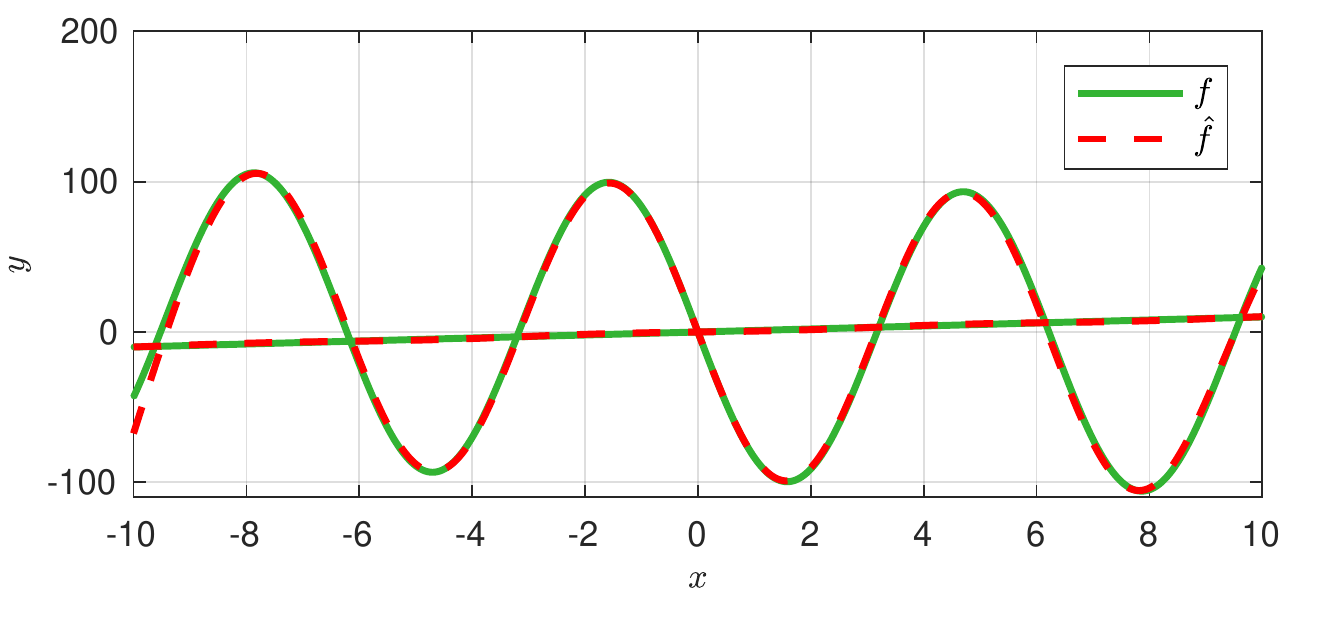}}
\subfigure[Learned and locally approximated dynamics]{
\includegraphics[width=0.45\linewidth]{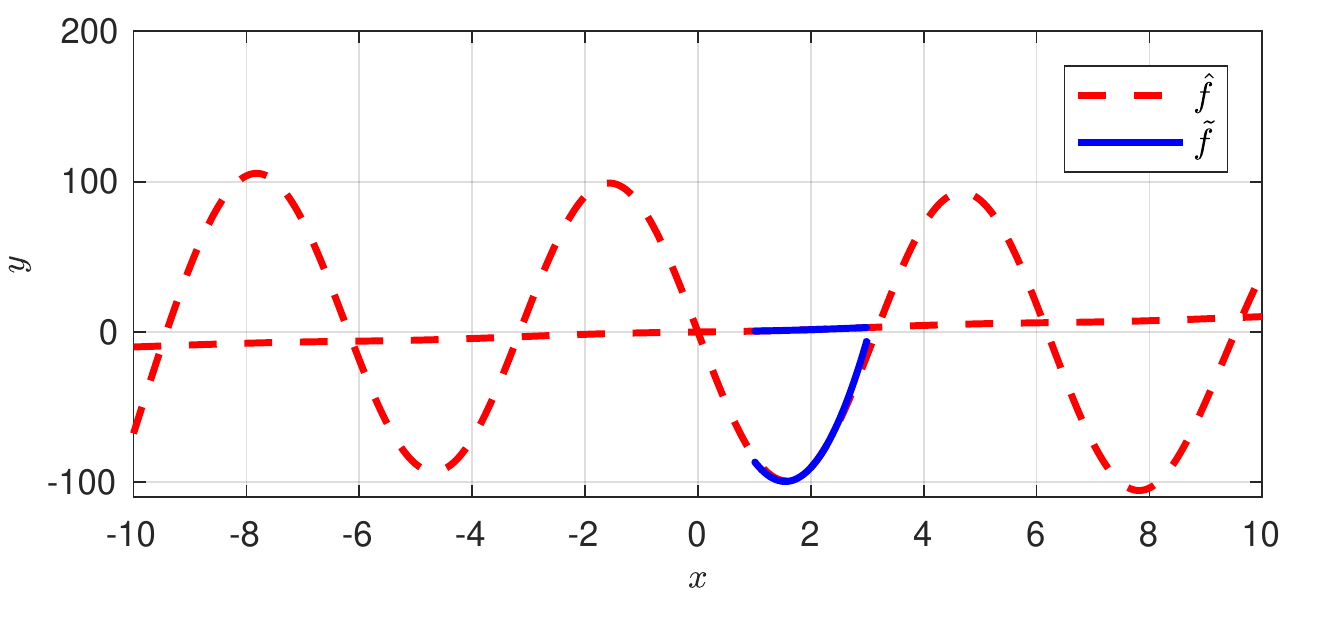}}\\
\caption{(a) shows the dynamics function for two dimensional dynamics of a pendulum. The dynamics functions of two dimensions are drawn on a mutual domain where both $x_1$ and $x_2$ ranges over $[-10, 10]$ (b) Taylor approximation of the learned dynamics function around working point $[x_1; x_2]=[2,  2]$.}
\label{fig:pendulum_true_vs_learned_2d}
\end{figure}

\begin{figure}[t!]
	\centering
	\hspace{-0.3cm}
\subfigure[1 section]{
\includegraphics[width=0.23\linewidth]{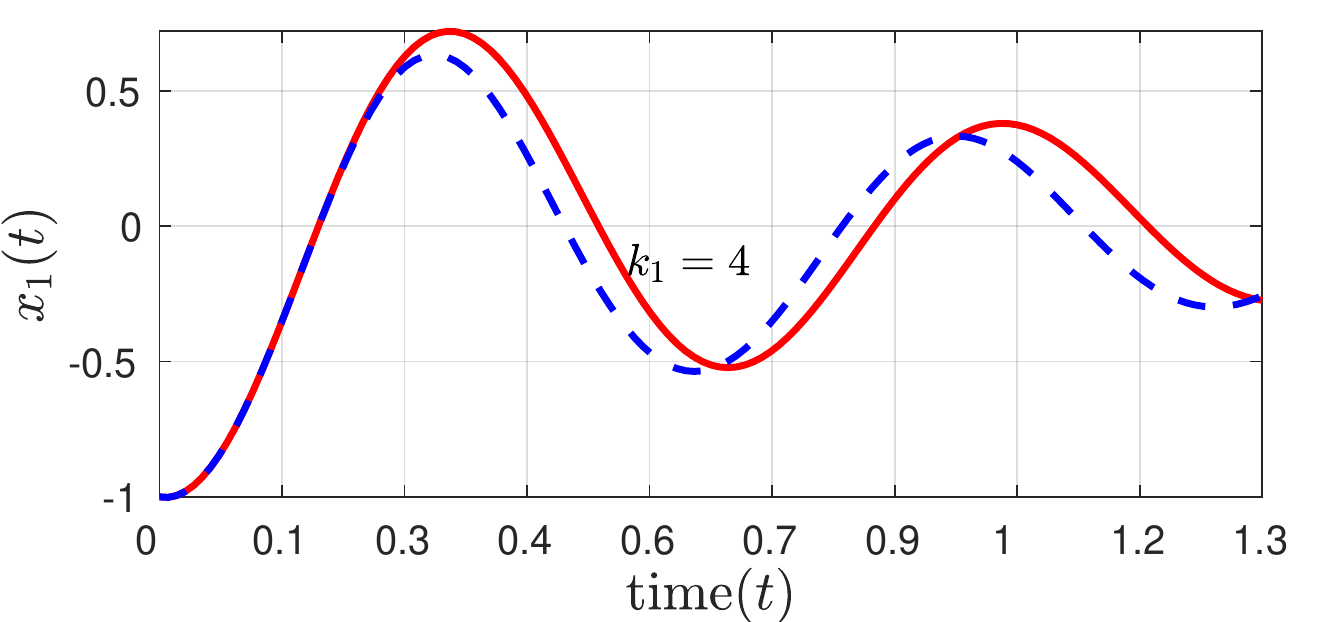}}
\subfigure[2 section]{
\includegraphics[width=0.23\linewidth]{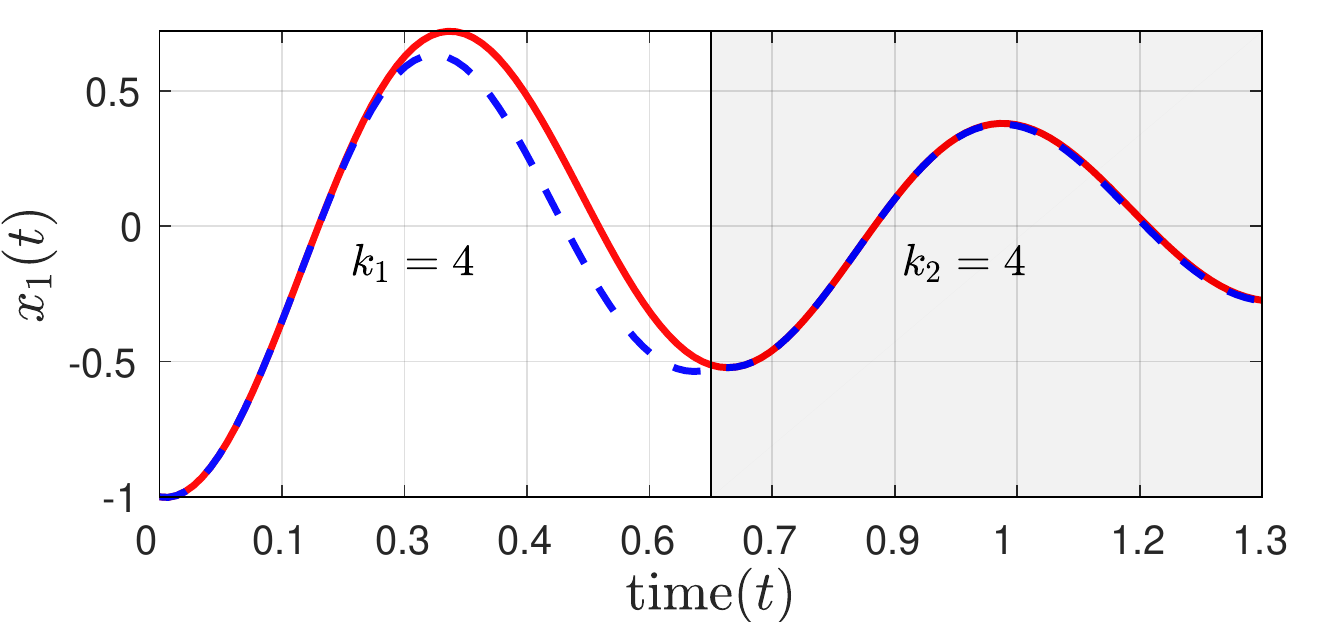}}
\subfigure[3 sections]{
\includegraphics[width=0.23\linewidth]{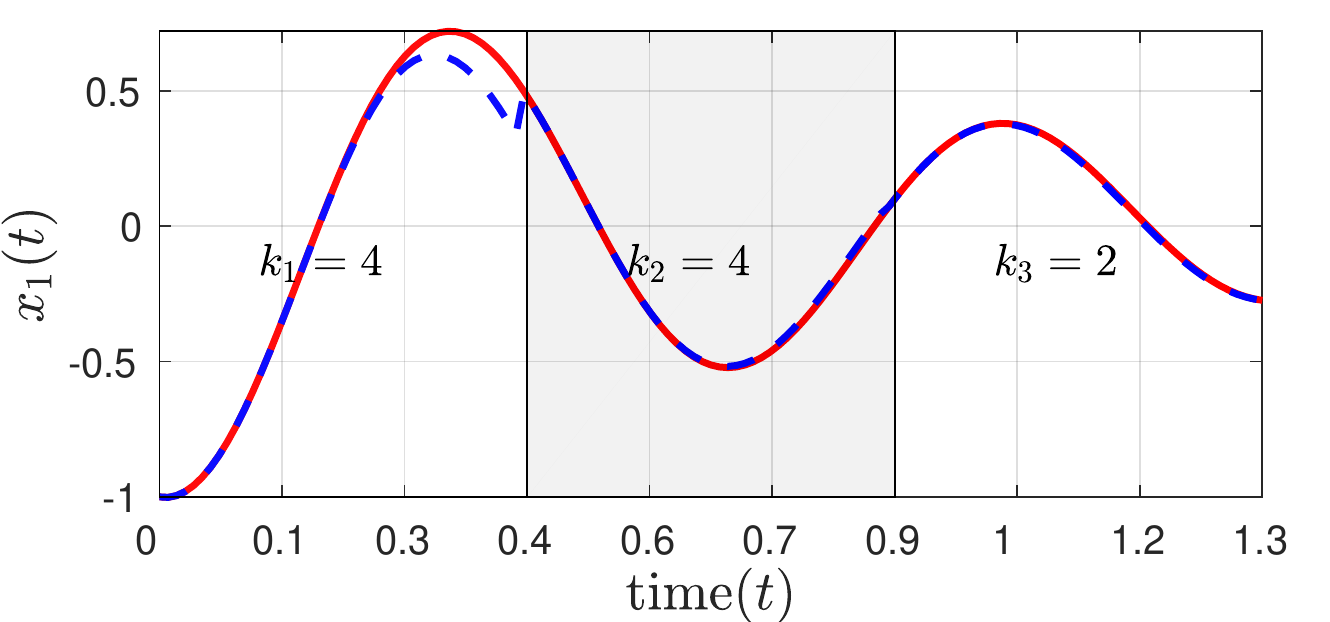}}
\subfigure[4 sections]{
\includegraphics[width=0.23\linewidth]{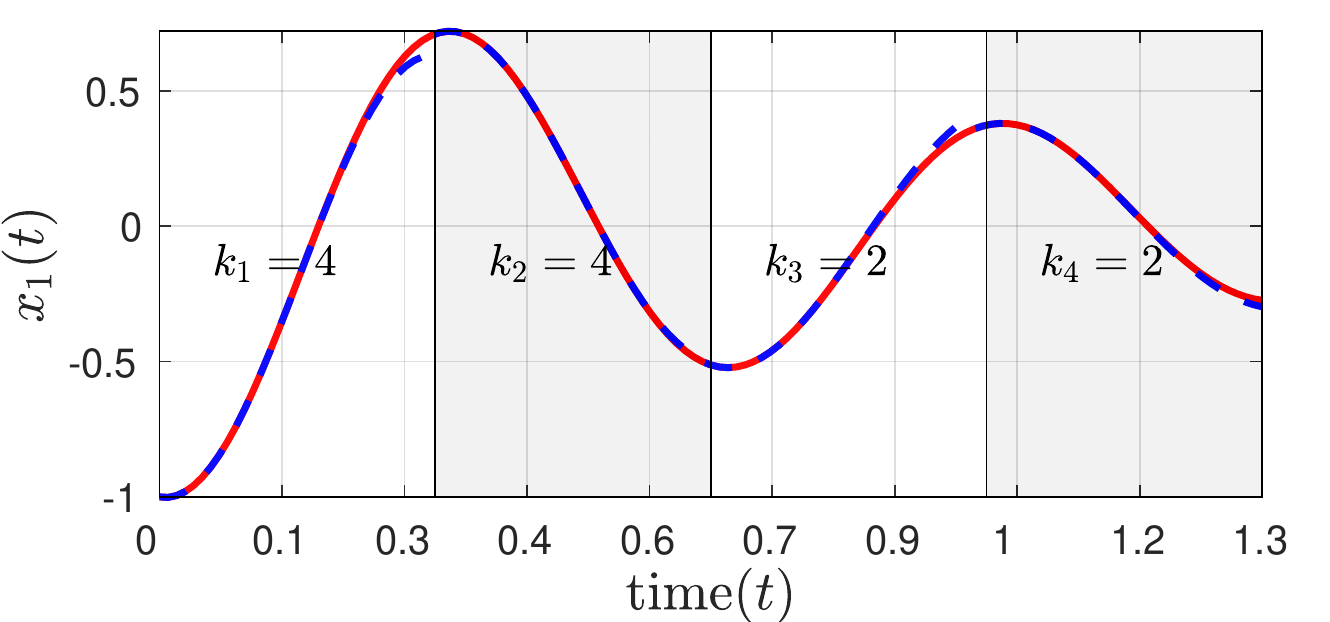}}\\
% \subfigure[4 sections]{
% \includegraphics[width=0.15\linewidth]{./figures_arxiv/nn_part_5_state_1_2d.pdf}}\\
\subfigure[1 section]{
\includegraphics[width=0.23\linewidth]{./figures_arxiv/nn_part_1_state_2_2d.pdf}}
\subfigure[2 section]{
\includegraphics[width=0.23\linewidth]{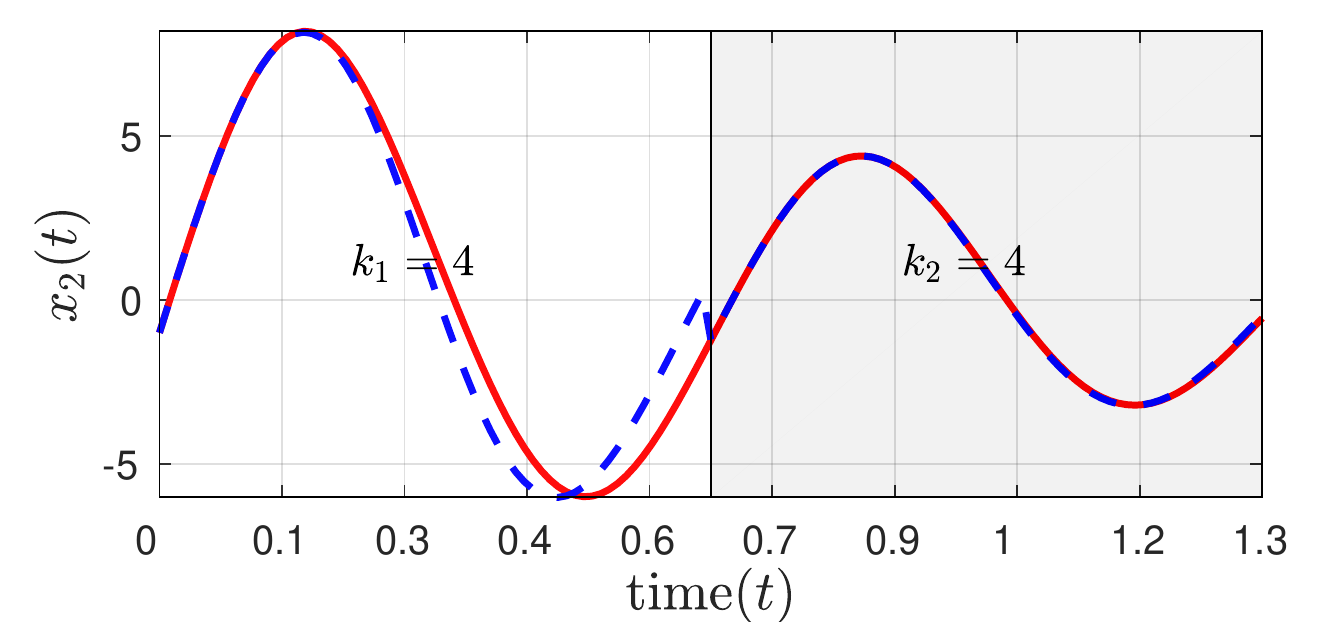}}
\subfigure[3 sections]{
\includegraphics[width=0.23\linewidth]{./figures_arxiv/nn_part_3_state_2_2d.pdf}}
\subfigure[4 sections]{
\includegraphics[width=0.23\linewidth]{./figures_arxiv/nn_part_4_state_2_2d.pdf}}
% \subfigure[4 sections]{
% \includegraphics[width=0.15\linewidth]{./figures_arxiv/nn_part_5_state_2_2d.pdf}}
\caption{State evolution of a two dimensional physical pendulum dynamical system with various number of local partitions. The minimum value of $L_{\rm total}$ where $L(m=\{1,2,3,4\})=[0.067, 0.048, 0035, 0.044]$ occurs when $m=3$ as a good trade-off between accuracy of the states and total complexity of the local models. The number of bases used to approximate each local region is also depicted over the figures.}
\vspace{-0.2cm}
\label{fig:exp_pendulum}
\end{figure}

In \fig \ref{fig:quadrotor}, we can see the effects of different model complexities on the state prediction accuracy. We take polynomials up to degree $5$ as allowed basis functions and the maximum number of partitions is upper bounded by $5$. Hence $k_{\rm max}=5$ in \alg 2, and $m_{\rm max}=5$ in \alg 1. Again, we see a non-trivial optimum, which is obtained through the principled MML approach.
\begin{figure}[t!]
	\centering
	\hspace{-0.3cm}
\subfigure[roll $\phi$]{
\includegraphics[width=0.3\linewidth]{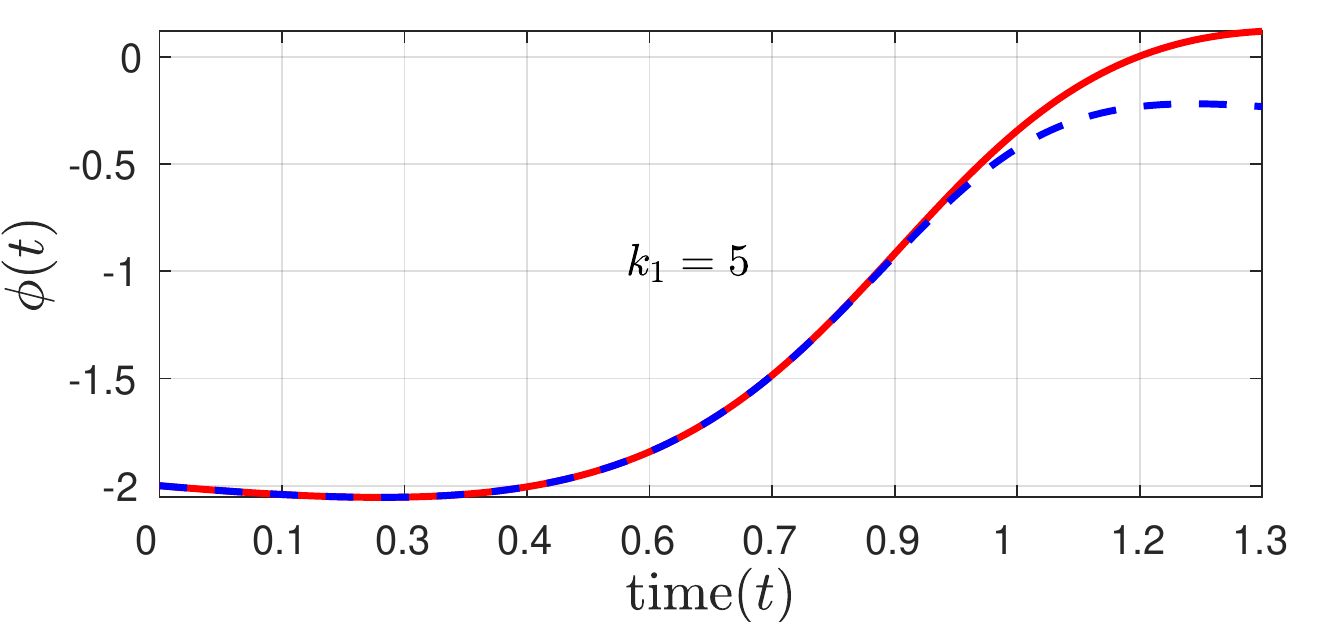}}
\subfigure[roll $\phi$]{
\includegraphics[width=0.3\linewidth]{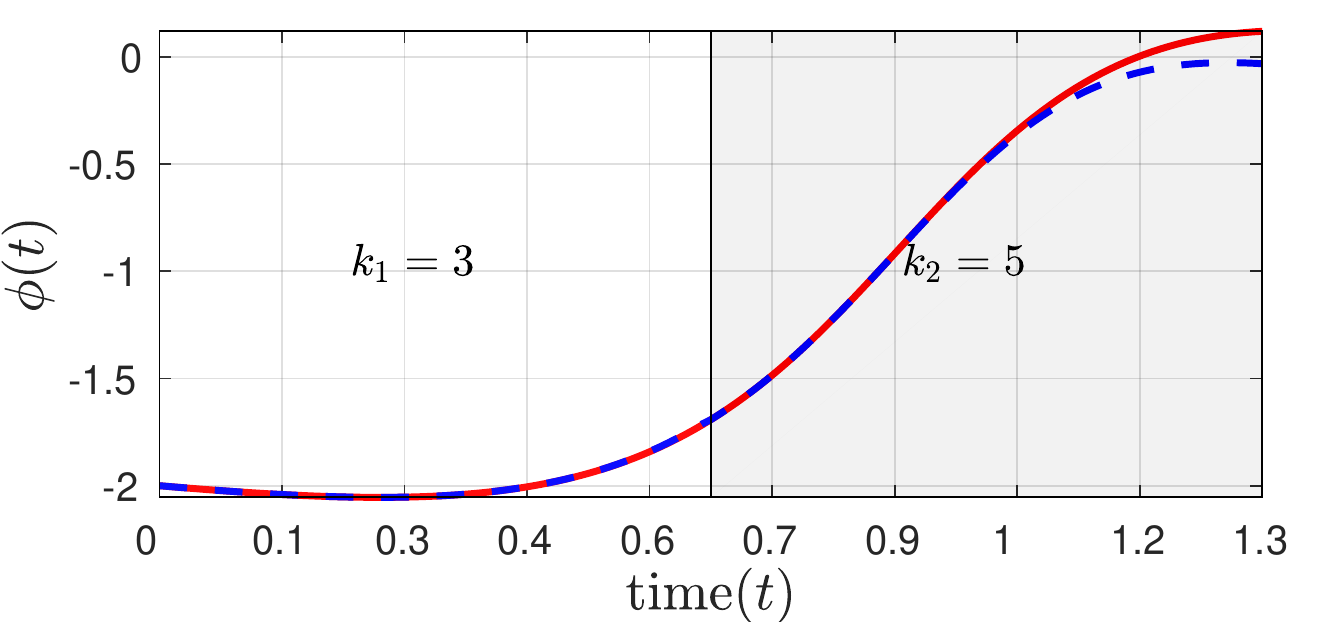}}
\subfigure[velocity $p$]{
\includegraphics[width=0.3\linewidth]{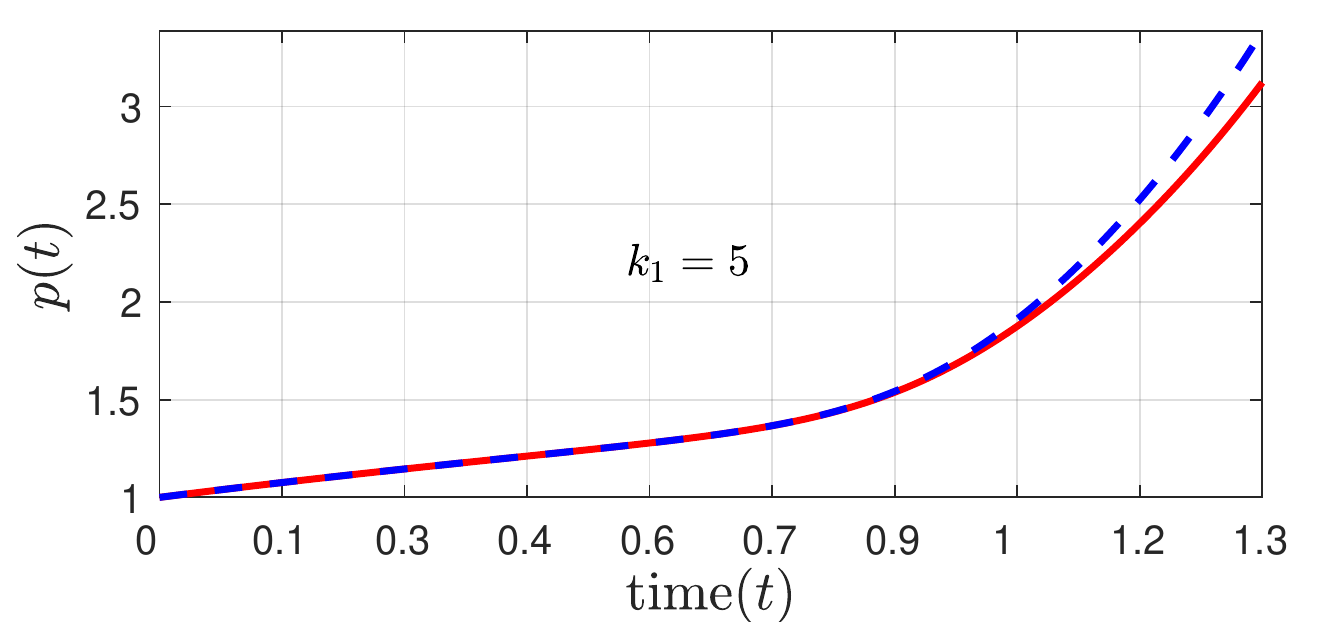}}\\
\subfigure[velocity $p$]{
\includegraphics[width=0.45\linewidth]{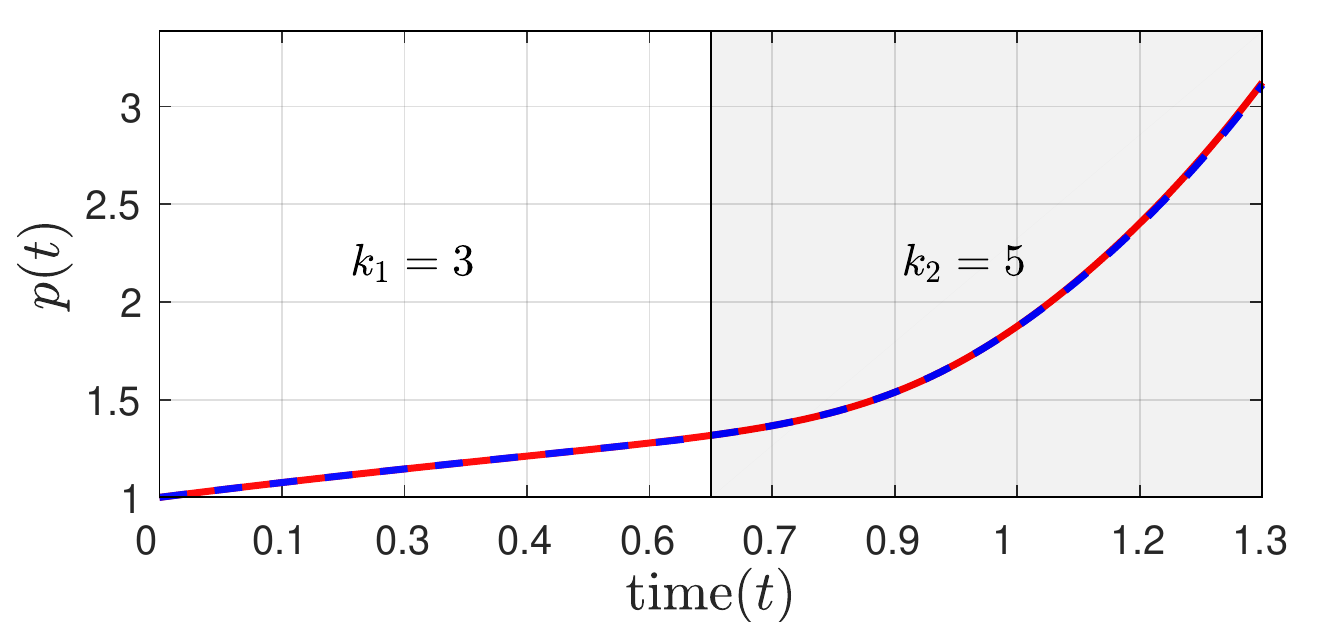}}
\subfigure[Cost function]{
\includegraphics[width=0.45\linewidth]{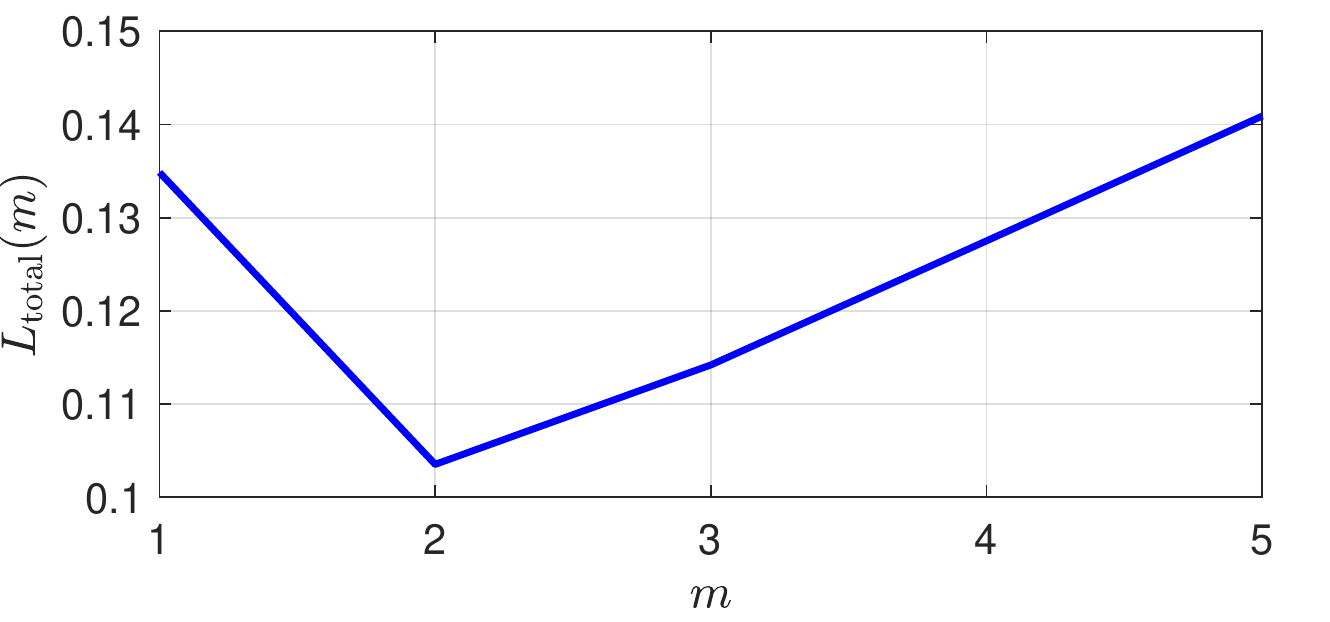}}
	\caption{Quadrotor dynamics. Top four figures show the effect of suboptimal ($m=1$, left) and optimal number of partitions ($m=2$, right) on two different states of the system. As can be seen, having single partition in the state space imposes a large value of the prediction error when the system evolves in time. Partitioning the state space into two local regions gives a fair trade-off between model complexity and total inaccuracy in the prediction of states. The diagram of (e) shows this trade-off due to the first and second terms of $L_{\rm total}$ with parameters $\lambda=2$ and $T_{\rm global}=2$ and sampling time step $\Delta T=0.01$. The overlaying numbers in each area show the number of Taylor terms chosen to model the function restricted to that area.\vspace{-0.2cm}}
	\vspace{-0.2cm}
\label{fig:quadrotor}
\end{figure}

\subsection{Dynamical system of Lorenz chaotic system}
We test the model selection capability of \licds for Lorenz system which is a three dimensional chaotic system described below. 
\begin{equation}
\left\{
	\begin{array}{ll}
	\dot{x}(t)=\sigma(y-x)\\
	\dot{y}(t)=x(\rho-z) - y\\
	\dot{z}(t)=xy-\beta z
	\end{array}
\right.
\end{equation}
with parameters $[\sigma, \beta, \rho]=[10, 8/3, 28]$ for which the system exhibits chaotic behaviour.
We used two neural networks with the same number of neurons but different architectures to learn the dynamics of the system. Fig.~\ref{fig:lorenz_model_selection} refers to the neural network that achieves a closer approximation to the dynamics. The value of \licds score for both architectures are computed. As can be seen in the subcaptions, the optimal value of the \licds score for the correct architecture (top row) that occurs for $m=2$ is less than the optimum value of \licds for the incorrect architecture (bottom row) that occurs for $m=2$. This confirmed our hypothesis that the best encoding by MML principle is a sign of the correct joint description of the model and data. We tested the accuracy of the learned $\hat{f}$ for a left-out subset of $S_c$ as the validation set and observed that the least validation error is concurrent with the least \licds score. This confirms the link between generalization and compression.

\makeatletter
\setlength{\@fptop}{0pt}
\makeatother

\begin{figure}[!]
\vspace{-10cm}
	\centering
	\hspace{-0.3cm}
\subfigure[{$L=0.035$, $E=23.83$}]{
\includegraphics[width=0.3\linewidth]{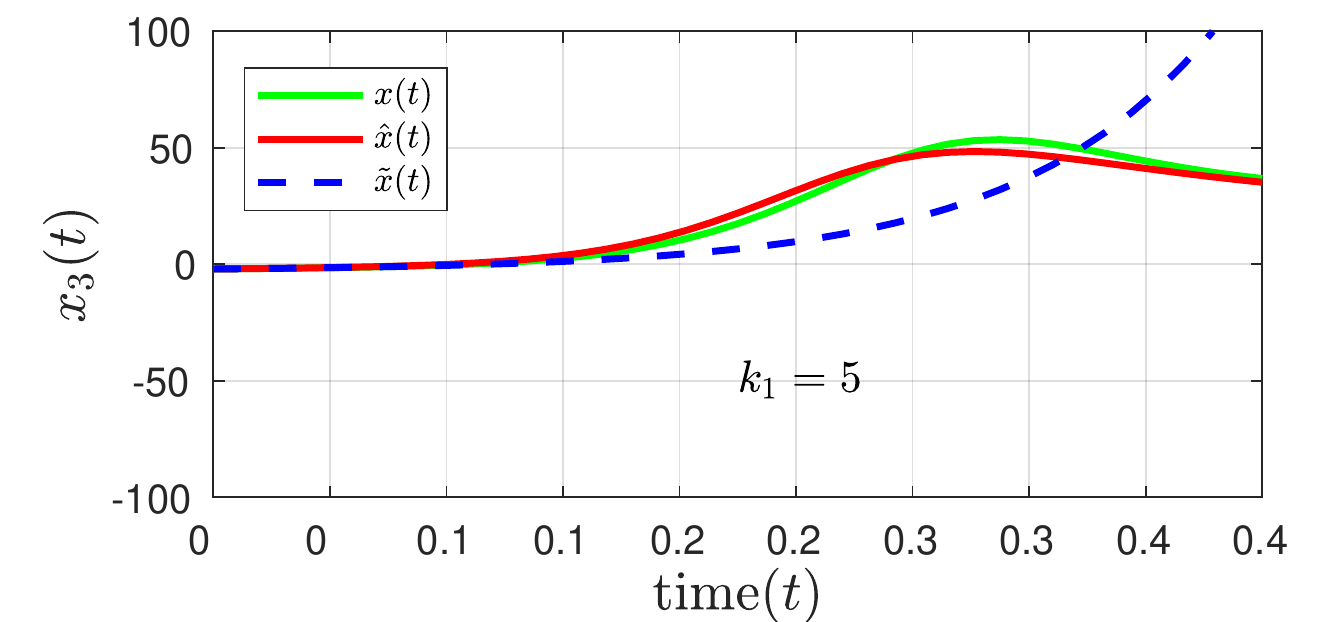}}
\subfigure[{$L=0.024$, $E=19.83$}]{
\includegraphics[width=0.3\linewidth]{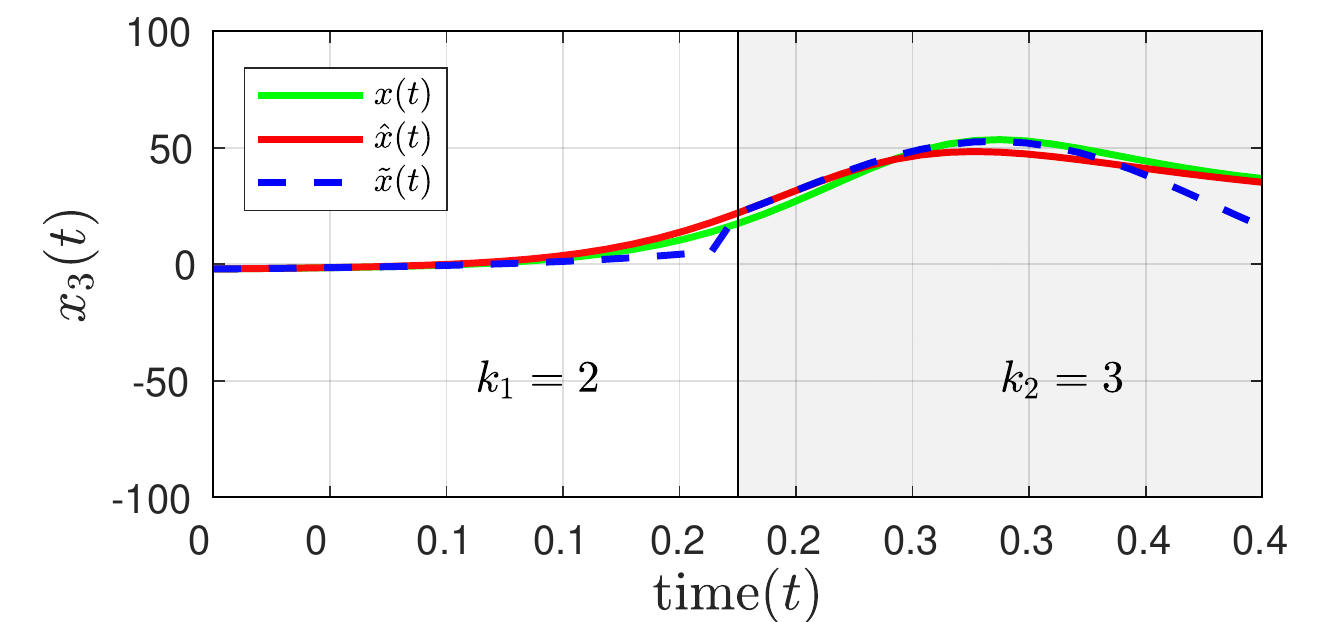}}
\subfigure[{$L=0.028$, $E=21.72$}]{
\includegraphics[width=0.3\linewidth]{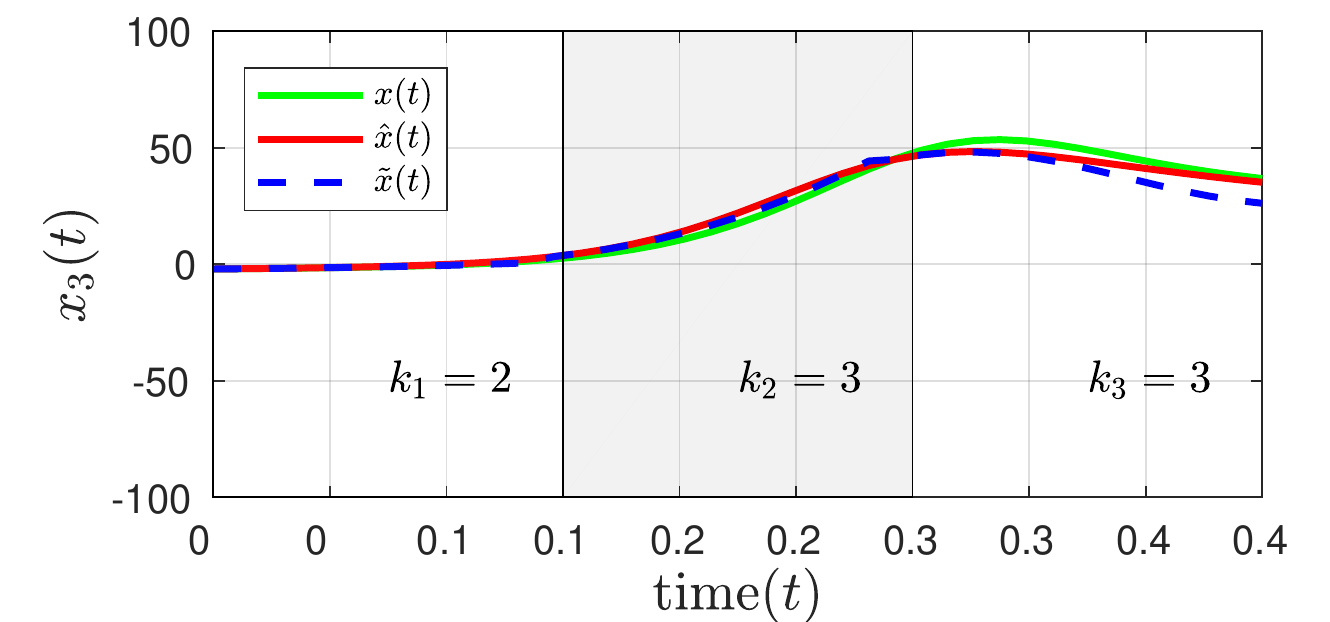}}\\
\subfigure[{$L=0.062$, $E=48.68$}]{
\includegraphics[width=0.3\linewidth]{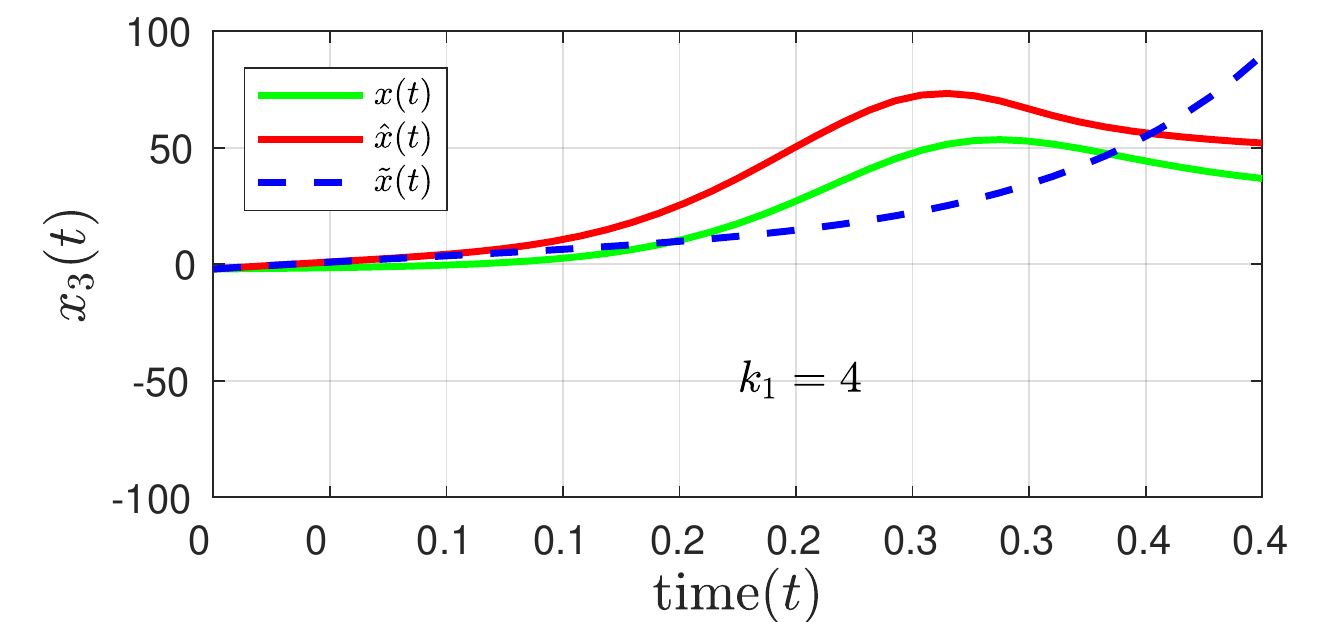}}
\subfigure[{$L=0.056$, $E=35.01$}]{
\includegraphics[width=0.3\linewidth]{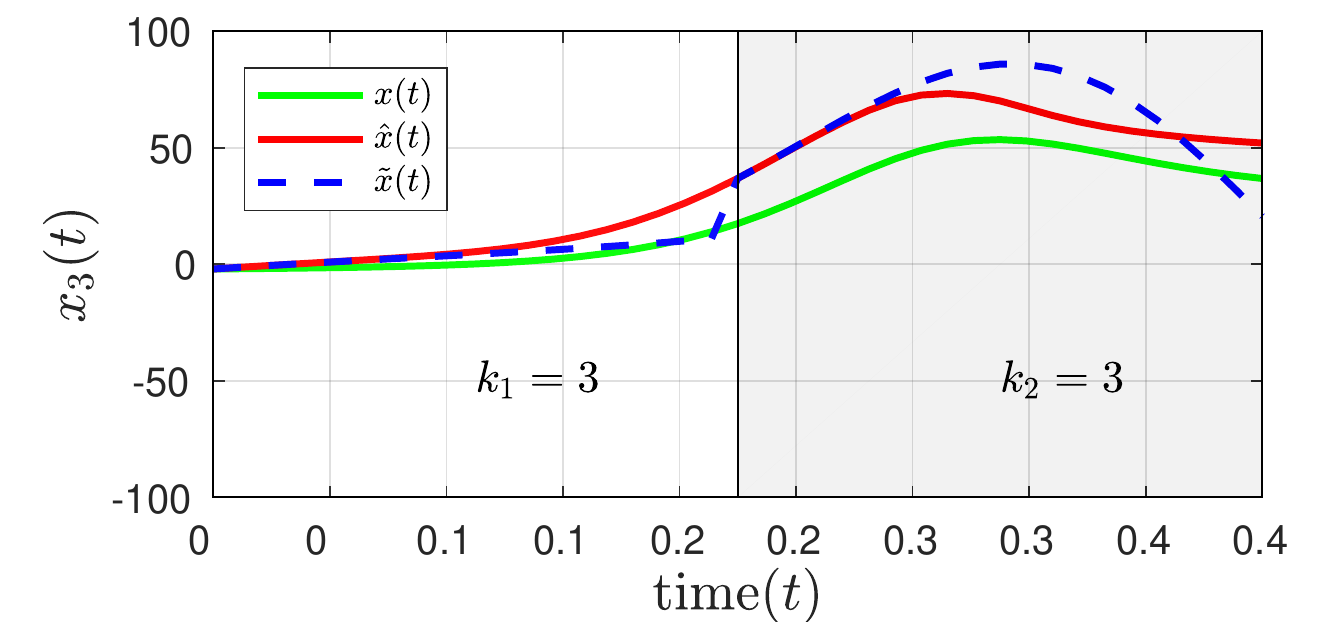}}
\subfigure[{$L=0.088$, $E=57.32$}]{
\includegraphics[width=0.3\linewidth]{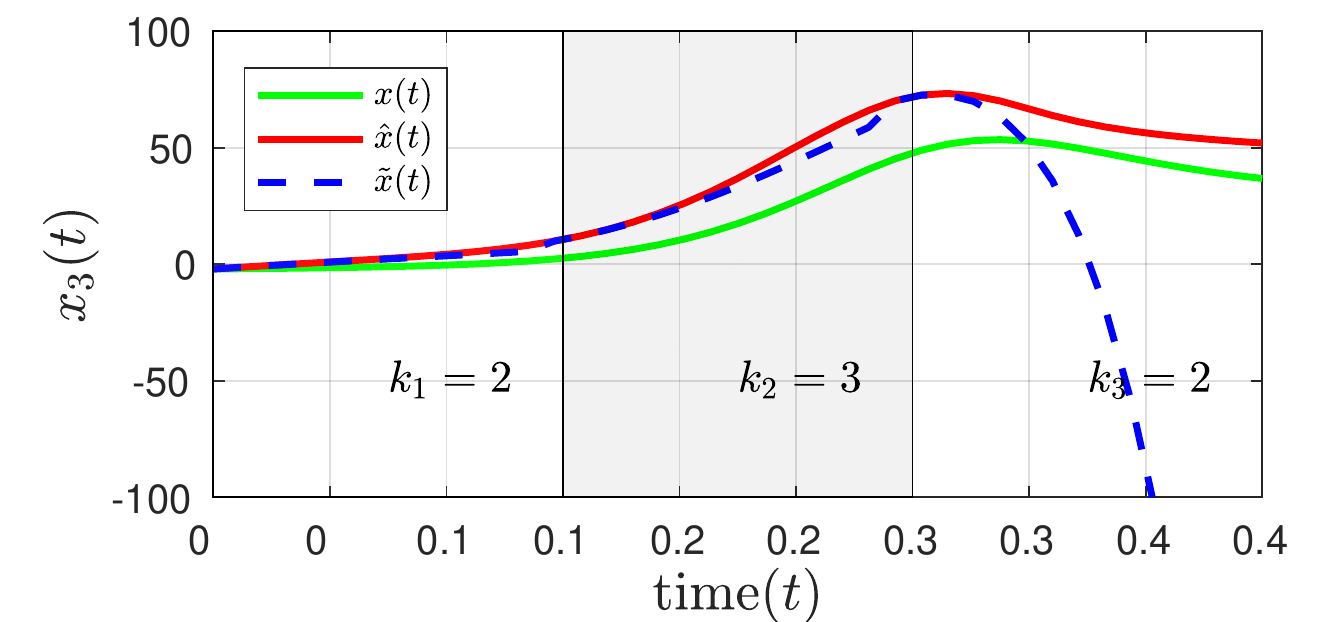}}\\
% \subfigure[4 sections]{
% \includegraphics[width=0.15\linewidth]{./figures_arxiv/nn_part_5_state_2_2d.pdf}}
\caption{Each figure shows the states of the true dynamics $f$ (solid), states of the learned dynamics $\hat{f}$ and states of the locally approximated dynamics $\tilde{f}$ (blue dashed). The neural network architecture NN = [50] for the top row and NN = [10, 10] for the bottom row.}
\vspace{-0.2cm}
\label{fig:lorenz_model_selection}
\end{figure}

\end{document}